\documentclass[11pt]{article}

\usepackage[colorlinks=true,linkcolor=blue,urlcolor=black,citecolor=blue,breaklinks]{hyperref}
\usepackage{fullpage}
\usepackage[T1]{fontenc}
\usepackage{color}
\usepackage{enumerate}
\usepackage{graphicx}
\usepackage{algorithm,algcompatible}
\usepackage{mathrsfs}
\usepackage{mathtools}
\usepackage[font={small}]{caption}

\usepackage{amsmath,amsthm,amssymb,colonequals,etoolbox}

\usepackage[square,numbers]{natbib}
\newtoggle{draft}
\togglefalse{draft}

\newtheorem{thm}{Theorem}[section]

\newtheorem{lem}[thm]{Lemma}

\newtheorem{cor}[thm]{Corollary}

\numberwithin{equation}{section}

\DeclareMathOperator*{\rank}{rank}

\DeclareMathOperator*{\Tr}{\mathbf{Tr}}

\DeclareMathOperator*{\Span}{span}
\DeclareMathOperator*{\polylog}{polylog}

\newcommand{\R}{\ensuremath{\mathbb{R}}}

\newcommand{\norm}[1]{\lVert #1 \rVert}
\newcommand{\bignorm}[1]{\left\lVert #1 \right\rVert}

\newcommand{\ip}[2]{\ensuremath{\langle #1, #2 \rangle}}
\newcommand{\bigip}[2]{\left\langle #1, #2 \right\rangle}

\newcommand{\Var}{\mathrm{Var}}
\newcommand{\Cov}{\mathrm{Cov}}
\newcommand{\E}{\mathbb{E}}
\newcommand{\abs}[1]{\ensuremath{| #1 |}}
\newcommand{\bigabs}[1]{\ensuremath{\left| #1 \right|}}

\newcommand{\ind}{\mathbf{1}}
\renewcommand{\vec}{\mathrm{vec}}
\newcommand{\svec}{\mathrm{svec}}
\newcommand{\mat}{\mathrm{mat}}
\newcommand{\smat}{\mathrm{smat}}

\renewcommand{\Pr}{\mathbb{P}}
\newcommand{\T}{\mathsf{T}}

\newcommand{\calA}{\mathcal{A}}

\newcommand{\calC}{\mathcal{C}}
\newcommand{\calN}{\mathcal{N}}
\newcommand{\calE}{\mathcal{E}}
\newcommand{\calG}{\mathcal{G}}

\newcommand{\calR}{\mathcal{R}}

\newcommand{\calT}{\mathcal{T}}

\newcommand{\cvectwo}[2]{\begin{bmatrix} #1 \\ #2 \end{bmatrix}}
\newcommand{\rvectwo}[2]{\begin{bmatrix} #1 & #2 \end{bmatrix}}
\newcommand{\bmattwo}[4]{\begin{bmatrix} #1 & #2 \\ #3 & #4 \end{bmatrix}}

\newcommand{\Ah}{\widehat{A}}
\newcommand{\Bh}{\widehat{B}}
\newcommand{\Lh}{\widehat{L}}
\newcommand{\Ph}{\widehat{P}}
\newcommand{\Kh}{\widehat{K}}
\newcommand{\Khplugin}{\Kh_{\mathrm{plug}}}
\newcommand{\Khpg}{\Kh_{\mathrm{pg}}}
\newcommand{\Thetah}{\widehat{\Theta}}
\newcommand{\Phplugin}{\Ph_{\mathrm{plug}}}
\newcommand{\Phlstd}{\Ph_{\mathrm{lstd}}}
\newcommand{\Astar}{A_\star}
\newcommand{\Bstar}{B_\star}

\newcommand{\Lstar}{L_\star}
\newcommand{\Kstar}{K_\star}

\newcommand{\Ustar}{U_\star}
\newcommand{\distconv}{\overset{D}{\rightsquigarrow}}
\newcommand{\asconv}{\overset{\mathrm{a.s.}}{\longrightarrow}}
\newcommand{\Ktilde}{\widetilde{K}}

\newcommand{\wlstd}{\widehat{w}_{\mathrm{lstd}}}
\newcommand{\dlyap}{\mathsf{dlyap}}
\newcommand{\lyap}{\mathsf{lyap}}

\newcommand{\thetatilde}{\tilde{\theta}}
\newcommand{\KL}{\mathrm{KL}}
\newcommand{\Otilde}{\widetilde{O}}
\newcommand{\iid}{\stackrel{\mathclap{\text{\scriptsize{ \tiny i.i.d.}}}}{\sim}}

\begin{document}

\title{The Gap Between Model-Based and Model-Free Methods on the Linear Quadratic Regulator: An Asymptotic Viewpoint}

\author{Stephen Tu and Benjamin Recht \\
University of California, Berkeley}
\maketitle

\begin{abstract}
The effectiveness of model-based versus model-free methods is a long-standing
question in reinforcement learning (RL). Motivated by recent empirical success
of RL on continuous control tasks, we study the sample complexity of popular
model-based and model-free algorithms on the Linear Quadratic Regulator (LQR).
We show that for policy evaluation, a simple model-based plugin method
requires asymptotically less samples than the classical least-squares temporal
difference (LSTD) estimator to reach the same quality of solution;
the sample complexity gap between the two methods can be at least a factor of state
dimension.
For policy evaluation, we study
a simple family of problem instances
and show that nominal (certainty equivalence principle) control
also requires several factors of state and input dimension fewer samples than
the policy gradient method
to reach the same level of control performance on these instances.
Furthermore, the gap persists even when employing commonly used baselines.
To the best of our knowledge, this is the first theoretical result which
demonstrates a separation in the sample complexity between model-based and model-free methods
on a continuous control task.
\end{abstract}

\section{Introduction}

The reinforcement learning (RL) community has been debating the relative merits of model-based and model-free methods for decades.
This debate has become reinvigorated in the last few years
due to the impressive success of RL techniques in various domains such as
game playing, robotic manipulation, and locomotion tasks. 
A common rule of thumb amongst RL practitioners is that model-free methods
have worse sample complexity compared to model-based methods, but
are generally able to achieve better performance asymptotically 
since they do not suffer from biases in the model that lead to sub-optimal behavior
\cite{clavera18,nagabandi18,pong18}.
However, there is currently no general theory 
which rigorously explains the gap between performance of model-based versus model-free methods.
While there has been theoretical work studying both model-based and model-free
methods in RL, prior work has primarily shown specific upper
bounds~\cite{agrawal17,azar17,jaksch10,jin18,strehl06}
which are not directly comparable, or information-theoretic
lower bounds~\cite{jaksch10,jin18} which are currently too coarse-grained to
delineate between model-based and model-free methods.
Furthermore, most of the prior work has focused primarily on the tabular Markov Decision Process (MDP) setting.

We take a first step towards a theoretical understanding of the 
differences between model-based and model-free methods for continuous control settings. 
While we are ultimately interested in comparing these methods
for general MDPs with non-linear state transition dynamics,
in this work we build upon recent progress in 
understanding the performance guarantees of data-driven methods
for the {Linear Quadratic Regulator} (LQR).
We study the asymptotic
behavior of both \emph{policy evaluation} and \emph{policy optimization} on LQR,
comparing the performance of simple model-based methods
which use empirical state transition data to fit a dynamics model versus
the performance of popular model-free methods from RL:
\emph{temporal-difference learning} for policy evaluation and 
\emph{policy gradient methods} for policy optimization.

Our analysis shows that
in the policy evaluation setting, a simple model-based plugin
estimator is always asymptotically more sample efficient than the classical 
least-squares temporal difference (LSTD) estimator;
the gap between the two methods can be at least a factor of state-dimension.
For policy optimization, we consider a simple family of instances
for which nominal control (also known as the certainty equivalence principle
in control theory) is also at least several factors of state and input dimension
more efficient than the widely used policy gradient method.
Furthermore, the gap persists
even when we employ commonly used baselines to reduce the variance
of the policy gradient estimate.
In both settings, we also show minimax lower bounds which highlight the
near-optimality of model-based methods on the family of instances we consider.
To the best of our knowledge, our work is the first to rigorously show a
setting where a strict separation between a model-based and model-free method
solving the same continuous control task occurs.

\section{Main Results}

In this paper, we study the performance of model-based and model-free algorithms
for the \emph{Linear Quadratic Regulator} (LQR) via two
fundamental primitives in reinforcement learning: \emph{policy evaluation} and \emph{policy optimization}.
In both tasks we fix an unknown dynamical system
\begin{align*}
    x_{t+1} = \Astar x_t + \Bstar u_t + w_t \:,
\end{align*}
starting at $x_0 = 0$ (for simplicity) and driven
by Gaussian white noise $w_t \iid \calN(0, \sigma_w^2 I_n)$. We let $n$ denote the state dimension and
$d$ denote the input dimension, and assume the system is underactuated (i.e. $d \leq n$). We also fix two positive semi-definite cost matrices $(Q, R)$.

\subsection{Policy Evaluation}
\label{sec:results:policy_eval}

Given a controller $K \in \R^{d \times n}$ that stabilizes $(\Astar, \Bstar)$,
the policy evaluation task is to compute the (relative) value function $V^K(x)$:
\begin{align}
    V^K(x) := \lim_{T \to \infty} \E\left[\sum_{t=0}^{T-1} (x_t^\T Q x_t + u_t^\T R u_t - \lambda_K) \:\bigg|\: x_0 = x \right] \:, \:\: u_t = K x_t \:.
\end{align}
Above, $\lambda_K$ is the infinite horizon average cost.
It is well-known that $V^K(x)$ can be written as:
\begin{align}
    V^K(x) = \sigma_w^2 x^\T P_\star x \:,
\end{align}
where $P_\star = \dlyap( \Astar + \Bstar K, Q + K^\T R K)$ solves the discrete-time Lyapunov equation:
\begin{align}
    (\Astar + \Bstar K)^\T P_\star (\Astar + \Bstar K) - P_\star + Q + K^\T R K = 0 \:. \label{eq:policy_eval_lyapunov}
\end{align}
From the Lyapunov equation, it is clear that given $(\Astar, \Bstar)$, the solution to
policy evaluation task is readily computable. In this paper, we study algorithms which
only have input/output access to $(\Astar, \Bstar)$. Specifically, we study \emph{on-policy} algorithms that operate on a \emph{single} trajectory,
where the input $u_t$ is determined by $u_t = K x_t$.
The variable that controls the amount of information available to the algorithm is $T$, the trajectory length.
The trajectory will be denoted as $\{ x_t \}_{t=0}^{T}$.
We are interested in the asymptotic behavior of algorithms as $T \to \infty$.

\paragraph{Model-based algorithm.}
In light of Equation~\eqref{eq:policy_eval_lyapunov}, the plugin estimator is a very natural model-based algorithm to use.
Let $\Lstar := \Astar + \Bstar K$ denote the true closed-loop matrix.
The plugin estimator uses the trajectory $\{ x_t \}_{t=0}^{T}$ to estimate $\Lstar$ via least-squares; call this $\Lh(T)$.
The estimator then returns $\Phplugin(T)$ by using $\Lh(T)$ in-place of $\Lstar$ in \eqref{eq:policy_eval_lyapunov}.
Algorithm~\ref{alg:model_based_policy_eval} describes this estimator.
\begin{center}
    \begin{algorithm}[htb]
    \caption{Model-based algorithm for policy evaluation.}
    \begin{algorithmic}[1]
        \REQUIRE{Policy $\pi(x) = K x$, rollout length $T$, regularization $\lambda > 0$, thresholds $\zeta \in (0, 1)$ and $\psi > 0$.}
        \STATE{Collect trajectory $\{ x_t \}_{t=0}^{T}$ using the feedback $u_t = \pi(x_t) = K x_t$.}
        \STATE{Estimate the closed-loop matrix via least-squares:
            \begin{align*}
                \Lh(T) = \left( \sum_{t=0}^{T-1} x_{t+1} x_t^\T \right) \left( \sum_{t=0}^{T-1} x_t x_t^\T + \lambda I_n \right)^{-1} \:.
            \end{align*}
        }
        \IF{$\rho(\Lh(T)) > \zeta$ or $\norm{\Lh(T)} > \psi$}
            \STATE{Set $\Phplugin(T) = 0$.}
        \ELSE
            \STATE{Set $\Phplugin(T) = \dlyap(\Lh(T), Q + K^\T R K)$.}
        \ENDIF
        \STATE{{\bf return} $\Phplugin(T)$.}
    \end{algorithmic}
    \label{alg:model_based_policy_eval}
    \end{algorithm}
\end{center}
\paragraph{Model-free algorithm.}
By observing that $V^K(x) = \sigma_w^2 x^\T P_\star x = \sigma_w^2
\ip{\svec(P_\star)}{\svec(xx^\T)}$, one can apply Least-Squares Temporal
Difference Learning (LSTD)~\cite{boyan99,bradtke96} with the feature map $\phi(x) :=
\svec(xx^\T)$ to estimate $P_\star$.
Here, $\svec(\cdot)$ vectorizes the upper triangular part of a symmetric matrix,
weighting the off-diagonal terms by $\sqrt{2}$ to ensure consistency in the inner product.
This is a classical algorithm in RL;
the pseudocode is given in Algorithm~\ref{alg:model_free_policy_eval}.
\begin{center}
    \begin{algorithm}[htb]
    \caption{Model-free algorithm for policy evaluation (LSTD)~\cite{bradtke96}.}
    \begin{algorithmic}[1]
        \REQUIRE{Policy $\pi(x) = K x$, rollout length $T$.}
        \STATE{Collect trajectory $\{ x_t \}_{t=0}^{T}$ using the feedback $u_t = \pi(x_t) = K x_t$.}
        \STATE{Estimate $\lambda_t \approx \sigma_w^2 \Tr(P_\star)$ from $\{x_t\}_{t=0}^{T}$.}
        \STATE{Compute (recall that $\phi(x) = \svec(xx^\T)$):
            \begin{align*}
                \wlstd(T) = \left( \sum_{t=0}^{T-1} \phi(x_t) (\phi(x_t) - \phi(x_{t+1}))^\T \right)^{-1} \left( \sum_{t=0}^{T-1} (c_t - \lambda_t) \phi(x_t) \right) \:,
            \end{align*}
        }
        \STATE{Set $\Phlstd(T) = \smat(\wlstd(T))$}.
        \STATE{{\bf return} $\Phlstd(T)$.}
    \end{algorithmic}
    \label{alg:model_free_policy_eval}
    \end{algorithm}
\end{center}
We now proceed to compare the risk of Algorithm~\ref{alg:model_based_policy_eval}
versus Algorithm~\ref{alg:model_free_policy_eval}.
Our notion of risk will be the expected squared error of the estimator: $\E[ \norm{ \widehat{P} - P_\star }_F^2 ]$.
Our first result gives an upper bound on the asymptotic risk of the model-based plugin Algorithm~\ref{alg:model_based_policy_eval}.
\begin{thm}
\label{thm:policy_eval_plugin_risk}
Let $K$ stabilize $(\Astar, \Bstar)$.
Define $\Lstar$ to be the closed-loop matrix $\Astar + \Bstar K$
and let $\rho(\Lstar) \in (0, 1)$ denote its stability radius.
Recall that $P_\star$ is the solution to the discrete-time Lyapunov equation \eqref{eq:policy_eval_lyapunov}
that parameterizes the value function $V^K(x)$.
We have that Algorithm~\ref{alg:model_based_policy_eval} with thresholds
$(\zeta, \psi)$ satisfying
$\zeta \in (\rho(\Lstar), 1)$ and $\psi \in (\norm{\Lstar}, \infty)$
and any fixed regularization parameter $\lambda > 0$
has the asymptotic risk upper bound:
\begin{align*}
    &\lim_{T \to \infty} T \cdot \E[ \norm{\Phplugin(T) - P_\star}_F^2 ] \leq 4 \Tr( (I - \Lstar^\T \otimes_s \Lstar^\T)^{-1}(\Lstar^\T P_\star^2 \Lstar \otimes_s \sigma_w^2P_\infty^{-1}) (I - \Lstar^\T \otimes_s \Lstar^\T)^{-\T}) \:.
\end{align*}
Here, $P_\infty = \dlyap(\Lstar^\T, \sigma_w^2 I_n)$ is the stationary covariance matrix of the closed-loop system
$x_{t+1} = \Lstar x_t + w_t$ and $\otimes_s$ denotes the symmetric Kronecker product.
\end{thm}
We make a few quick remarks regarding Theorem~\ref{thm:policy_eval_plugin_risk}.
First, while the risk bound is presented as an upper bound,
the exact asymptotic risk can be recovered from the proof.
{Second, the thresholds $(\zeta, \psi)$ and regularization parameter $\lambda$
do not affect the final asymptotic bound, but do possibly affect both higher
order terms and the rate of convergence to the limiting risk.  We include these
thresholds as they simplify the proof.  In practice, we find that thresholding
or regularization is generally not needed, with the caveat that if the estimate
$\Lh(T)$ is not stable then the solution to the discrete Lyapunov equation is
not guaranteed to exist (and when it exists is not guaranteed to be positive semidefinite).}
Finally, we remark that a non-asymptotic high probability upper bound for the risk of Algorithm~\ref{alg:model_based_policy_eval}
can be easily derived by combining the single trajectory learning results of \citet{simchowitz18}
with standard results on perturbation of Lyapunov equations.

We now turn our attention to the model-free LSTD algorithm.
Our next result gives a lower bound on the asymptotic risk of Algorithm~\ref{alg:model_free_policy_eval}.
\begin{thm}
\label{thm:policy_eval_lstd_risk}
Let $K$ stabilize $(\Astar, \Bstar)$.
Define $\Lstar$ to be the closed-loop matrix $\Astar + \Bstar K$.
Recall that $P_\star$ is the solution to the discrete-time Lyapunov equation \eqref{eq:policy_eval_lyapunov}
that parameterizes the value function $V^K(x)$.
We have that Algorithm~\ref{alg:model_free_policy_eval} with the cost estimates $\lambda_t$
set to the true cost $\lambda_\star := \sigma_w^2 \Tr(P_\star)$
satisfies the asymptotic risk lower bound:
\begin{align*}
  &\liminf_{T \to \infty} T \cdot \E[ \norm{ \Phlstd(T) - P_\star}_F^2 ] \geq 4 \calR_{\mathsf{plug}} \\
  &\qquad+ 8 \sigma_w^2 \ip{P_\infty}{\Lstar^\T P_\star^2 \Lstar} \Tr((I - \Lstar^\T \otimes_s \Lstar^\T)^{-1} (P_\infty^{-1} \otimes_s P_\infty^{-1})(I - \Lstar^\T \otimes_s \Lstar^\T)^{-\T})
\end{align*}
Here,
$\calR_{\mathsf{plug}} := \lim_{T \to \infty} T \cdot \E[ \norm{\Phplugin(T) - P_\star}_F^2 ]$
is the asymptotic risk of the plugin estimator,
$P_\infty = \dlyap(\Lstar^\T, \sigma_w^2 I_n)$ is the stationary covariance matrix of the closed loop system
$x_{t+1} = \Lstar x_t + w_t$, and
$\otimes_s$ denotes the symmetric Kronecker product.
\end{thm}
Theorem~\ref{thm:policy_eval_lstd_risk} shows that the asymptotic risk of the
model-free method always exceeds that of the model-based plugin method.
We remark that we prove the theorem under an idealized setting where the
infinite horizon cost estimate $\lambda_t$ is set to the true cost $\lambda_\star$.
In practice, the true cost is not known and must instead be estimated
from the data at hand.
However, for the purposes of our comparison this is not an issue
because using the true cost $\lambda_\star$ over an estimator of $\lambda_\star$
only reduces the variance of the risk.

To get a sense of how much excess risk is incurred by the model-free
method over the model-based method,
consider the following family of instances, defined for $\rho \in (0, 1)$ and $1 \leq d \leq n$:
\begin{align}
    \mathscr{F}(\rho, d, K) := \{ (\Astar, \Bstar) : \Astar + \Bstar K = \tau P_E + \gamma I_n \:, \: (\tau, \gamma) \in (0, 1) \:, \: \tau + \gamma \leq \rho \:, \: \dim(E) \leq d \} \:.
\end{align}
With this family, one can show with elementary computations that
under the simplifying assumptions that $Q + K^\T R K = I_n$ and $d \asymp n$,
Theorem~\ref{thm:policy_eval_plugin_risk}
and Theorem~\ref{thm:policy_eval_lstd_risk} state that:
\begin{align*}
  \lim_{T \to \infty} T \cdot \E[ \norm{\Phplugin(T) - P_\star}_F^2 ] &\leq O\left( \frac{\rho^2 n^2}{(1-\rho^2)^3} \right) \:, \\
  \liminf_{T \to \infty} T \cdot \E[ \norm{\Phlstd(T) - P_\star}_F^2 ] &\geq \Omega\left( \frac{\rho^2 n^3}{(1-\rho^2)^3} \right) \:.
\end{align*}
That is, for $\mathscr{F}(\rho, d, K)$,
the plugin risk is a factor of state-dimension $n$ less than the LSTD risk.
Moreover, the non-asymptotic result for LSTD from Lemma 4.1 of \citet{abbasi18}
(which extends the non-asymptotic \emph{discounted} LSTD result from \citet{tu18a})
gives a bound of $\norm{\Phlstd(T) - P_\star}_F^2 \leq \Otilde(n^3/T)$ w.h.p., which matches
the asymptotic bound of Theorem~\ref{thm:policy_eval_lstd_risk} in terms of $n$ up to logarithmic factors.

Our final result for policy evaluation is a minimax lower bound on
the risk of any estimator over $\mathscr{F}(\rho, d, K)$.

\begin{thm}
\label{thm:policy_eval_lower_bound}
Fix a $\rho \in (0, 1)$ and suppose that $K$ satisfies $Q + K^\T R K = I_n$.
Suppose that $n$ is greater than an absolute constant and
$T \gtrsim n (1-\rho^2)/\rho^2$. We have that:
\begin{align*}
  \inf_{\Ph} \sup_{(\Astar, \Bstar) \in \mathscr{F}(\rho, \frac{n}{4}, K)} \E[ \norm{\Ph - P_\star}_F^2 ] \gtrsim \frac{\rho^2 n^2}{(1-\rho^2)^3 T } \:,
\end{align*}
where the infimum is taken over all estimators $\Ph$ taking input $\{x_t\}_{t=0}^{T}$.
\end{thm}
Theorem~\ref{thm:policy_eval_lower_bound} states that the rate achieved by
the model-based Algorithm~\ref{alg:model_based_policy_opt} over the family $\mathscr{F}(\rho, d, K)$
cannot be improved
beyond constant factors, at least asymptotically;
its dependence on both the state dimension $n$ and stability radius $\rho$ is
optimal.

\subsection{Policy Optimization}
\label{sec:results:policy_opt}

Given a finite horizon length $T$, the policy optimization task is to solve
the finite horizon optimal control problem:
\begin{align}
    J_\star := \min_{u_t(\cdot)} \E\left[ \sum_{t=0}^{T-1} (x_t^\T Q x_t + u_t^\T R u_t) + x_T^\T Q x_T \right] \:, \:\: x_{t+1} = \Astar x_t + \Bstar u_t + w_t \:. \label{eq:finite_horizon_problem}
\end{align}
We will focus on a special case of this problem when there is no penalty on the input: $Q = I_n$, $R = 0$,
and $\mathrm{range}(\Astar) \subseteq \mathrm{range}(\Bstar)$.
In this situation, the cost function reduces to $\E[ \sum_{t=0}^{T} \norm{x_t}_2^2 ]$
and
the optimal solution simply chooses a $u_t$ that cancels
out the state $x_t$;
that is $u_t = K_\star x_t$ with $K_\star := - \Bstar^{\dag} \Astar$.
We work with this simple class of instances so that we can ensure that policy gradient
converges to the optimal solution; in general this is not guaranteed.

We consider a slightly different input/output oracle model in this setting than we did in Section~\ref{sec:results:policy_eval}.
The horizon length $T$ is now considered fixed, and $N$ rounds are played.
At each round $i=1, ..., N$, the algorithm chooses a feedback matrix $K_i \in \R^{d \times n}$.
The algorithm then observes the trajectory $\{x_t^{(i)}\}_{t=0}^{T}$ by playing the
control input $u_t^{(i)} = K_i x_t^{(i)} + \eta_t^{(i)}$, where $\eta_t^{(i)} \sim \calN(0, \sigma_u^2 I_d)$ is i.i.d.\ noise used
for the policy. This process
then repeats for $N$ total rounds. After the $N$ rounds, the algorithm is asked
to output a $\Kh(N)$ and is assigned the risk $\E[ J(\Kh(N)) - J_\star ]$,
where $J(\Kh(N))$ denotes playing the feedback $u_t = \Kh(N) x_t$ on the true system $(\Astar, \Bstar)$.
We will study the behavior of algorithms when $N \to \infty$ (and $T$ is held fixed).

\paragraph{Model-based algorithm.}
Under this oracle model, a natural model-based algorithm is to
first use random open-loop feedback (i.e. $K_i = 0$) to observe $N$ independent
trajectories (each of length $T$), and then use the trajectory
data to fit the state transition matrices $(\Astar, \Bstar)$;
call this estimate $(\Ah(N), \Bh(N))$. After fitting the dynamics,
the algorithm then returns the estimate of $\Kstar$
by solving the finite horizon problem \eqref{eq:finite_horizon_problem} with $(\Ah(N), \Bh(N))$ taking the place
of $(\Astar, \Bstar)$.  In general, however, the assumption that
$\mathrm{range}(\Ah(N)) \subseteq \mathrm{range}(\Bh(N))$ will not hold, and hence
the optimal solution to \eqref{eq:finite_horizon_problem} will not be time-invariant.
Moreover, solving for the best time-invariant static feedback for the finite horizon problem
in general is not tractable. In light of this, to provide the fairest comparison to the model-free policy gradient method,
we use the time-invariant static feedback that arises from
infinite horizon solution given by the discrete algebraic Riccati equation as a proxy.
We note that under our range inclusion assumption, the infinite horizon solution is a consistent estimator of the optimal feedback.
The pseudo-code for this model-based algorithm is described in Algorithm~\ref{alg:model_based_policy_opt}.
\begin{center}
    \begin{algorithm}[htb]
    \caption{Model-based algorithm for policy optimization.}
    \begin{algorithmic}[1]
        \REQUIRE{Horizon length $T$, rollouts $N$, regularization $\lambda$, thresholds $\varrho \in (0, 1), \zeta, \psi, \gamma$.}
        \STATE{Collect trajectories $\{ \{ ( x_t^{(i)}, u_t^{(i)} ) \}_{t=0}^{T} \}_{i=1}^{N}$ using the feedback $K_i = 0$ (open-loop).}
        \STATE{Estimate the dynamics matrices $(\Astar, \Bstar)$ via regularized least-squares:
            \begin{align*}
              \Thetah(N) = \left( \sum_{i=1}^{N} \sum_{t=0}^{T-1} x_{t+1} (z_t^{(i)})^\T \right) \left( \sum_{i=1}^{N} \sum_{t=0}^{T-1} z_t^{(i)} (z_t^{(i)})^\T + \lambda I_{n+d} \right)^{-1} \:, \:\: z_t^{(i)} := (x_t^{(i)}, u_t^{(i)}) \:.
            \end{align*}
        }
        \STATE{Set $(\Ah, \Bh) = \Thetah(N)$.}
        \IF{$\rho(\Ah) > \varrho$ or $\norm{\Ah} > \zeta$ or $\norm{\Bh} > \psi$ or $\sigma_d(\Bh) < \gamma$}
          \STATE{Set $\Khplugin(N) = 0$.}
        \ELSE
          \STATE{Set $\widehat{P} = \mathsf{dare}(\Ah, \Bh, I_n, 0)$ as the positive definite solution to\footnotemark:
            \begin{align*}
                P = \Ah^\T P \Ah - \Ah^\T P \Bh (\Bh^\T P \Bh)^{-1} \Bh^\T P \Ah + I_n \:.
            \end{align*}
          }
          \STATE{Set $\Khplugin(N) = - (\Bh^\T \widehat{P} \Bh)^{-1} \Bh^\T \widehat{P} \Ah$.}
        \ENDIF
        \STATE{{\bf return} $\Khplugin(N)$.}
    \end{algorithmic}
    \label{alg:model_based_policy_opt}
    \end{algorithm}
\end{center}
\footnotetext{
A sufficient condition
for the existence of a unique positive definite solution to the discrete algebraic Riccati equation when $R=0$
is that $(A, B)$ is stabilizable and $B$ has full column rank (Lemma~\ref{lem:dare_existence}).}

\paragraph{Model-free algorithm.}
We study a model-free algorithm based on policy gradients (see e.g.~\cite{peters08,williams92}).
Here, we choose to parameterize the policy as a time-invariant linear feedback.
The algorithm is described in Algorithm~\ref{alg:model_free_policy_opt}.
\begin{center}
    \begin{algorithm}[htb]
    \caption{Model-free algorithm for policy optimization (REINFORCE) \cite{peters08,williams92}.}
    \begin{algorithmic}[1]
      \REQUIRE{Horizon length $T$, rollouts $N$, baseline functions $\{\Psi_t(\cdot;\cdot)\}$, step-sizes $\{\alpha_i\}$, initial $K_1$, threshold $\zeta$.}
      \FOR{$i=1, ..., N$}
        \STATE{ Collect trajectory $\calT^{(i)} := \{ (x_t^{(i)}, u_t^{(i)} ) \}_{t=0}^{T}$ using feedback $K_i$. }
        \STATE{ Compute policy gradient $g_i$ as:
            $g_i = \frac{1}{\sigma_u^2} \sum_{t=0}^{T-1} \eta_t^{(i)} (x_t^{(i)})^\T \Psi_t(\calT^{(i)}; K_i)$.
        }
        \STATE{ Take policy gradient step: $K_{i+1} = \mathsf{Proj}_{\norm{\cdot} \leq \zeta}(K_i - \alpha_i g_i)$.
        }
      \ENDFOR
      \STATE{Set $\Khpg(N) = K_N$.}
      \STATE{{\bf return} $\Khpg(N)$.}
    \end{algorithmic}
    \label{alg:model_free_policy_opt}
    \end{algorithm}
\end{center}

In general for problems with a continuous action space,
when applying policy gradient one has many degrees of freedom in choosing how to represent the policy
$\pi$. Some of these degrees of freedom include whether or not the policy should be time-invariant
and how much of the history before time $t$ should be used to compute the action at time $t$.
More broadly, the question is what function class should be used to model the policy.
Ideally, one chooses a function class which is both capable of expressing the optimal solution
and is easy to optimize over.

Another issue that significantly impacts the performance of policy gradient
in practice is choosing a baseline which effectively reduces the variance of
the policy gradient estimate. What makes computing a baseline challenging
is that good baselines (such as value or advantage functions)
require knowledge of the unknown MDP transition dynamics in order to compute.
Therefore, one has to estimate the baseline from the empirical trajectories, adding
another layer of complexity to the policy gradient algorithm.

In general, these issues are still an active area of research in RL
and present many hurdles to a general theory for policy optimization.
However, by restriction our attention to LQR, we can sidestep these issues
which enables our analysis.
In particular, by studying problems with no penalty on the input and where the
state can be cancelled at every step, we know that
the optimal control is a static time-invariant linear feedback.
Therefore, we
can restrict our policy representation to static linear feedback controllers
without introducing any approximation error.
Furthermore, it turns out that the specific assumptions on $(\Astar, \Bstar)$
that we impose imply that the optimization landscape satisfies a standard
notion of restricted strong convexity.
This allows us to study policy gradient by leveraging the existing theory on the asymptotic distribution of
stochastic gradient descent for strongly convex objectives.
Finally, we can compute many of the standard baselines used in closed form, which further enables our analysis.

We note that in the literature, the model-based method is often called \emph{nominal control}
or the \emph{certainty equivalence principle}. As noted in \citet{dean17},
one issue with this approach is that on an infinite horizon,
there is no guarantee of robust stability with nominal control.
However, as we are dealing with only finite horizon problems, the notion of stability is irrelevant.

Our first result for policy optimization
gives the asymptotic risk of the model-based Algorithm~\ref{alg:model_based_policy_opt}.
\begin{thm}
\label{thm:policy_opt_plugin_risk}
Let $(\Astar, \Bstar)$ be such that $\Astar$ is stable, $\mathrm{range}(\Astar) \subseteq \mathrm{range}(\Bstar)$, and $\Bstar$ has full column rank.
We have that the model-based plugin Algorithm~\ref{alg:model_based_policy_opt} with
thresholds $(\varrho, \zeta, \psi, \gamma)$ such that $\varrho \in (\rho(\Astar), 1)$,
$\zeta \in (\norm{\Astar}, \infty)$,
$\psi \in (\norm{\Bstar}, \infty)$,
and $\gamma \in (0, \sigma_{d}(\Bstar))$ satisfies the asymptotic risk bound:
\begin{align*}
    \lim_{N \to \infty} N \cdot \E[J(\Khplugin(N)) - J_\star] = O( d (\Tr(P_\infty^{-1}) + \norm{K_\star}_F^2) ) + o_T(1) \:.
\end{align*}
Here, $P_\infty = \dlyap(\Astar, \sigma_u^2 \Bstar\Bstar^\T + \sigma_w^2 I_n)$ is the steady-state convariance
of the system driven with control input $u_t \sim \calN(0, \sigma_u^2 I_d)$, $K_\star$ is the optimal controller, and
$O(\cdot)$ hides constants depending only on $\sigma_w^2, \sigma_u^2$.
\end{thm}
%
%
We can interpret Theorem~\ref{thm:policy_opt_plugin_risk} by upper bounding $P_\infty^{-1} \preceq \sigma_w^{-2} I_n$.
In this case if $\norm{K_\star}_F^2 \leq O(n)$, then this result states that the asymptotic risk scales as
$O(nd/N)$.
{Similar to Theorem~\ref{thm:policy_eval_plugin_risk},
Theorem~\ref{thm:policy_opt_plugin_risk} requires the setting of thresholds
$(\varrho, \zeta, \psi, \gamma)$. These thresholds serve two purposes.
First, they ensure the existence of a unique positive definite solution to
the discrete algebraic Riccati solution with the input penalty $R=0$
(the details of this are worked out in Section~\ref{sec:pg:policy_opt_plugin_risk}).
Second, they simplify various technical aspects of the proof related to uniform integrability.
In practice, such strong thresholds are not needed, and we leave either removing them or relaxing their
requirements to future work.}

Next, we look at the model-free case.
As mentioned previously, baselines are very influential on the behavior of policy gradient.
In our analysis, we consider three different baselines:
\begin{align*}
    \Psi_t(\calT; K) &= \sum_{\ell=t+1}^{T} \norm{x_\ell}^2_2 \:, && \text{(\textbf{Simple} baseline $b_t(x_t;K) = \norm{x_t}^2_2$.)} \\
    \Psi_t(\calT; K) &= \sum_{\ell=t}^{T} \norm{x_\ell}^2_2 - V^{K}_t(x_t) \:, && \text{(\textbf{Value function} baseline $b_t(x_t;K) = V^K_t(x_t)$.)} \\
    \Psi_t(\calT; K) &= A^{K}_t(x_t, u_t) \:. && \text{(\textbf{Advantage} baseline $A^K_t(x_t, u_t) = Q^K_t(x_t, u_t) - V^K_t(x_t)$.)}
\end{align*}
Above, the simple baseline should be interpreted as having effectively no baseline;
it turns out to simplify the variance calculations.
On the other hand, the value function baseline $V_t^K$ is a very popular heuristic used
in practice~\cite{peters08}. Typically one has to actually estimate the value function
for a given policy, since computing it requires knowledge of the model dynamics.
In our analysis however, we simply assume the true value function is known.
While this is an unrealistic assumption
in practice, we note that this assumption substantially reduce the variance of policy gradient,
and hence only serves to reduce the asymptotic risk.
The last baseline we consider is to use the advantage function $A_t^K$.
Using advantage functions has been shown to be quite effective in practice~\cite{schulman16}.
It has the same issue as the value function baseline in that it needs to be estimated
from the data; once again in our analysis
we simply assume we have access to the true advantage function.

Our main result for model-free policy optimization is the following
asymptotic risk lower bound on Algorithm~\ref{alg:model_free_policy_opt}.
\begin{thm}
\label{thm:policy_opt_pg_risk}
Let $(\Astar, \Bstar)$ be such that $\Astar$ is stable, $\mathrm{range}(\Astar) \subseteq \mathrm{range}(\Bstar)$, and $\Bstar$ has full column rank.
Consider
Algorithm~\ref{alg:model_free_policy_opt} with $K_1 = 0_{d \times n}$,
step-sizes $\alpha_i = [2 (T-1) \sigma_w^2 \sigma_{d}(\Bstar)^2 \cdot i]^{-1}$, and threshold $\zeta \in (\norm{\Kstar}, \infty)$.
We have that the risk
is lower
bounded by:
\begin{align*}
    &\liminf_{N \to \infty} N \cdot \E[J(\Khpg(N)) - J_\star]  \geq \frac{1}{\sigma_{d}(\Bstar)^2 (1 + \norm{\Bstar}^2 ) } \times \\
     &\qquad\begin{cases}
        \Omega( T^2 d (n+\norm{\Bstar}_F^2)^3 ) + o_T(T^2) &\text{ (Simple baseline)} \\
        \Omega(T d( n + \norm{\Bstar}_F^2)(n + \norm{\Bstar^\T \Bstar}_F^2)) + o_T(T) &\text{ (Value function baseline)} \\
        \Omega(d (n + \norm{\Bstar}_F^2) \norm{\Bstar^\T \Bstar}_F^2 ) &\text{ (Advantage baseline)} \\
    \end{cases} \:.
\end{align*}
Here, $\Omega(\cdot)$ hides constants depending only on $\sigma_w^2, \sigma_u^2$.
\end{thm}
In order to interpret Theorem~\ref{thm:policy_opt_pg_risk},
we consider a restricted family of instances $(\Astar, \Bstar)$.
For a $\rho \in (0, 1)$ and $1 \leq d \leq n$, we define the
family $\mathscr{G}(\rho, d)$ over $(\Astar, \Bstar)$ as:
\begin{align*}
  \mathscr{G}(\rho, d) := \{ (\rho \Ustar\Ustar^\T, \rho \Ustar) : \Ustar \in \R^{n \times d} \:, \:\: \Ustar^\T \Ustar = I_d \} \:.
\end{align*}
This is a simple family where the $\Astar$ matrix is
stable and contractive, and furthermore we have $\mathrm{range}(\Astar) = \mathrm{range}(\Bstar)$.
The optimal feedback is $\Kstar = -\Ustar^\T$ for each of these instances.

Theorem~\ref{thm:policy_opt_pg_risk} states that for instances from $\mathscr{G}(\rho, d)$, the
simple baseline has risk $\Omega(T^2 \cdot d n^3 / N)$, the value function baseline has risk
$\Omega(T \cdot d n^2 / N)$, and the advantage baseline has risk $\Omega(d^2 n / N)$.
On the other hand, Theorem~\ref{thm:policy_opt_plugin_risk} states that the model-based
risk is upper bounded by $O(n d / N)$, which is less than the lower bound
for all baselines considered in Theorem~\ref{thm:policy_opt_pg_risk}.
For the simple and value function baselines, we see that
the sample complexity of the model-free policy gradient method
is several factors of $n$ and $T$ more than the model-based method.
The extra factors of the horizon
length appear due to the large variance of the policy gradient estimator
without the variance reduction effects of the advantage baseline.
The advantage baseline performs the best, only one factor of $d$ more than the
model-based method.

{We note that we prove Theorem~\ref{thm:policy_opt_pg_risk}
with a specific choice of step size $\alpha_i$.
This step size corresponds to the standard $1/(mt)$ step sizes commonly found
in proofs for SGD on strongly convex functions (see e.g.~\citet{rakhlin12}), where
$m$ is the strong convexity parameter. We leave to future work
extending our results to support Polyak-Ruppert averaging, which would yield
asymptotic results that are more robust to specific step size choices.}

Finally, we turn to our information-theoretic lower bound for any (possibly adaptive) method
over the family $\mathscr{G}(\rho, d)$.
\begin{thm}
\label{thm:policy_opt_lower_bound}
Fix a $d \leq n/2$ and suppose $d(n-d)$ is greater than an absolute constant.
Consider the family $\mathscr{G}(\rho, d)$ as describe above.
Fix a time horizon $T$ and number of rollouts $N$.
The risk over any algorithm $\calA$ which plays (possibly adaptive) feedbacks of the form $u_t = K_i x_t + \eta_t$
with $\norm{K_i} \leq 1$ and $\eta_t \sim \calN(0, \sigma_u^2 I_d)$ is lower bounded by:
\begin{align*}
    \inf_{\calA} \sup_{\substack{\rho \in (0, 1/4),\\ (\Astar, \Bstar) \in \mathscr{G}(d, \rho)}} \E[ J(\calA) - J_\star ] \gtrsim \frac{\sigma_w^4}{\sigma_w^2 + \sigma_u^2} \frac{d(n-d)}{N}\:.
\end{align*}
\end{thm}
Observe that this bound is $\Omega(nd/N)$. Therefore,
Theorem~\ref{thm:policy_opt_lower_bound} tells us that asymptotically, the
model-based method in Algorithm~\ref{alg:model_based_policy_opt} is
optimal in terms of its dependence on the state and input dimensions $n$ and $d$
over the family $\mathscr{G}(\rho, d)$.

\section{Related Work}

For general Markov Decision Processes (MDPs), 
the setting which is the best understood theoretically is the finite-horizon
episodic case with discrete state and action spaces, often referred to as the
``tabular'' setting. 
\citet{jin18} provide an excellent overview of the 
known regret bounds in the tabular setting; here we give a brief summary of the highlights.
We focus only on regret bounds for simplicity, but note that many results have also
been establishes in the PAC setting (see e.g.~\cite{lattimore14,strehl09,strehl06}).
For tabular MDPs, a model-based method is one which
stores the entire state-transition matrix, which takes $O(S^2 A H)$ space
where $S$ is the number of states, $A$ is the number of actions, and $H$ is the horizon length.
The best known regret bound in the model-free case is $\Otilde(\sqrt{H^2 SAT})$ from \citet{azar17},
which matches the known lower bound of $\Omega(\sqrt{H^2 SAT})$ from \citet{jaksch10,jin18} up to log factors.
On the other hand, a model-free method is one which only stores the $Q$-function and hence
requires only $O(S A H)$ space. The best known regret bound in the model-free case is
$\Otilde( \sqrt{H^3 S A T} )$, which is worse than the model-based case by a factor of the horizon length $H$.
Interestingly, there is no gap in terms of the number of states $S$ and actions $A$.
It is open whether or not the gap in $H$ is fundamental or can be closed.
\citet{sun18} present an information-theoretic definition of model-free algorithms.
Under their definition, they construct a family of factored MDPs with horizon length $H$
where any model-free algorithm incurs sample complexity $\Omega(2^H)$, whereas there exists
a model-based algorithm that has sample complexity polynomial in $H$ and other relevant quantities. 
We leave proving lower bounds for LQR under their more general definition of model-free algorithms to future work.

For LQR, the story is less complete. Unlike the tabular setting, the storage requirements
of a model-based method are comparable to a model-free method.
For instance, it takes $O(n(n+d))$ space to store the state transition model
and $O((n+d)^2)$ space to store the $Q$-function.
In presenting the known results of LQR, we will delineate between offline (one-shot) methods
versus online (adaptive) methods.

In the offline setting, the first non-asymptotic result 
is from \cite{fiechter97}, who studied the sample complexity of 
the \emph{discounted} infinite horizon LQR problem. Later, \citet{dean17}
study the average cost infinite horizon problem, using tools from
robust control to quantify how the uncertainty in the model affects control performance
in an interpretable way. Both works fall under model-based methods, since they both propose
to first estimate the state transition matrices from sample trajectories using least-squares
and then use the estimated dynamics in a control synthesis procedure.

For model-free methods for LQR, \citet{tu18a} study the performance
of least-squares temporal difference learning (LSTD)~\cite{boyan99,bradtke96}, which
is a classic policy evaluation algorithm in RL. They focus on the discounted cost LQR setting
and provide a non-asymptotic high probability bound on the risk of LSTD.
Later, \citet{abbasi18} extend this result to the average cost LQR setting.
Most related to our analysis for policy gradient is \citet{fazel18}, 
who study the performance of model-free policy gradient related methods
on LQR. Unfortunately, their bounds do not give explicit dependence on the problem
instance parameters and are therefore difficult to compare to. Furthermore, Fazel et al.\ study 
a simplified version of the problem where the problem is a infinite horizon problem (as opposed to finite horizon
in this work) and the only noise is in the
initial state; all subsequence state transitions have no process noise. 
Other than our current work, we are unaware of any analysis (asymptotic or non-asymptotic)
which explicitly studies the behavior of policy gradient on the finite horizon LQR problem.
We also note that Fazel et al.\ analyze a policy optimization method which
is more akin to derivative-free random search (e.g.~\cite{nesterov17,mania18,salimans17}) than REINFORCE.
Derivative-free random search for LQR is studied by \citet{malik18}, 
who prove upper bounds that suggest that having two point evaluations is more sample efficient
compared to single point evaluation.
We leave analyzing these derivative-free algorithms under our framework to future work.
Finally, note that all the results mentioned for LQR are only \emph{upper bounds}; we are 
unaware of any \emph{lower bounds} in the literature for LQR which give explicit
dependence on the problem instance.

We now discuss known results for the online (adaptive) setting for LQR.
For model-based algorithms, both
\emph{optimism in the face of uncertainty}
(OFU)~\cite{abbasi11,faradonbeh17b,ibrahimi12} and \emph{Thompson
sampling}~\cite{abeille17,abeille18,ouyang17} have been analyzed in the online
learning literature. In both cases, the algorithms have been shown to achieve
$\Otilde(\sqrt{T})$ regret, which is known to be nearly optimal in the
dependence on $T$. However, in nearly all the bounds the dependence on the
problem instance parameters is hidden.  Furthermore, it is currently unclear
how to solve the OFU subproblem in polynomial time for LQR.  In response to the
computational issues with OFU, \citet{dean18} propose a polynomial
time adaptive algorithm with sub-linear regret $\Otilde(T^{2/3})$; their bounds
also make the dependence on the problem instance parameters
explicit, but are quite conservative in this regard.

For model-free algorithms, \citet{abbasi18} study the
regret of a model-free algorithm similar in spirit to
least-squares policy iteration (LSPI)~\cite{lagoudakis03}.  They prove that
their algorithm has regret $\Otilde(T^{2/3+\varepsilon})$ for any $\varepsilon
> 0$, nearly matching the bound given by Dean et al.\ in terms of the dependence on
$T$.  In terms of the dependence on the problem specific parameters, however,
their bound is not directly comparable to that of Dean et al. 
Experimentally, Abbasi-Yadkori et al.\ observe that their model-free algorithm performs quite
sub-optimally compared to model-based methods; these empirical observations are also
consistent with similar experiments conducted in \cite{mania18,recht18,tu18a}.

\section{Proof Sketch}

At a high level our proofs are relatively straightforward, relying on classical arguments from
asymptotic statistics. However, various technical issues arise which make the arguments more involved.
Below, we briefly outline the proof strategies that we use for our main results.

\subsection{Policy Evaluation}

\paragraph{Model-based (Algorithm~\ref{alg:model_based_policy_eval} and Theorem~\ref{thm:policy_eval_plugin_risk}).}

We first compute the limiting distribution of $\sqrt{T} \vec(\Lh(T) - \Lstar)$ (Lemma~\ref{lem:ls_asymptotic_dist})
as a consequnce of Markov chain CLTs (Theorem~\ref{thm:markov_chain_CLT}). We then use the delta method to compute the limiting
distribution of $\sqrt{T} \svec(\Phplugin(T) - P_\star)$ by differentiating the map $L \mapsto \dlyap(L, Q + K^\T R K)$ at $\Lstar$.
We then prove that this sequence of random variables is uniformly integrable by controlling its higher order moments.
Uniform integrability then implies (Lemma~\ref{lem:lp_convergence}) that the limit of the scaled risk $T \cdot \E[ \norm{\Phplugin(T) - P_\star}_F^2 ]$ is equal to the trace of the covariance of this limiting distribution, which yields the result.
The details of this are worked out in Section~\ref{sec:proofs:policy_eval_plugin}.

\paragraph{Model-free LSTD (Algorithm~\ref{alg:model_free_policy_eval} and Theorem~\ref{thm:policy_eval_lstd_risk}).}

This case is simpler since we can directly compute the limiting distribution of 
$\sqrt{T} (\wlstd - w_\star)$ (Lemma~\ref{lem:lstd_asymptotic_distribution}) using Markov chain CLTs.
Then, the trace of the limiting distribution immediately lower bounds (Lemma~\ref{lem:lp_convergence}) the limit
of the scaled risk $T \cdot \E[ \norm{\Phlstd - P_\star}_F^2 ]$ without having to 
establish uniform integrability. The details are worked out in Section~\ref{sec:proofs:policy_eval_lstd}.

\subsection{Policy Optimization}

\paragraph{Model-based (Algorithm~\ref{alg:model_based_policy_opt} and Theorem~\ref{thm:policy_opt_plugin_risk}).}

As before, we start by computing the limiting distribution of $\sqrt{N} \vec(\Thetah(N) - \Theta_\star)$ (Lemma~\ref{lem:ls_driven_asymptotic_dist}).
Next, we use the delta method to compute the limiting distribution of $\sqrt{N} \vec(K(\Thetah(N)) - K_\star)$, where
$K(\Theta)$ is the optimal LQR controller designed with the model parameters $\Theta$. This is done by differentiating 
the solution of the discrete algebraic Riccati equation with respect to the model parameters (Lemma~\ref{lem:dare_derivative}).
Next, we make the observation that $J(K_\star)=0$ and apply the second order delta method in order to compute the limiting distribution of 
$N \cdot (J(\Kh(N)) - J_\star)$. We then show uniform integrability of this sequence by once again controlling its higher order moments.
Again, by Lemma~\ref{lem:lp_convergence} this yields an expression for the limit of the scaled risk $N \cdot \E[ J(\Kh(N)) - J_\star ]$.
The details are worked out in Section~\ref{sec:pg:policy_opt_plugin_risk}.

\paragraph{Model-free policy gradients (Algorithm~\ref{alg:model_free_policy_opt} and Theorem~\ref{thm:policy_opt_pg_risk}).}

This proof is the most involved, because it requires us to establish the limiting distribution of a first-order stochastic
optimization algorithm with a convex projection step. In Section~\ref{sec:pg:calcs}, we show that the geometry of the smoothed policy gradient function
satisfies restricted strong convexity for the particular dynamics we consider. We then (Section~\ref{sec:proof_sgd_asymptotics}) compute the limiting distribution for SGD with projection
on restricted strong convex functions, when the optimal solution lives in the interior of the domain. To do this, we build on the results of \citet{toulis17} and \citet{rakhlin12}.
The remainder of the proof involves computing the variance of the various policy gradient estimators with different baselines.
In general this calculation is not tractable, but our particular choice of models we study allows us to obtain very sharp estimates on this variance.
The details are worked out in Section~\ref{sec:pg:policy_opt_pg}.

\section{Conclusion}

We compared the asymptotic performance of both model-based and
model-free methods for LQR. We showed that for policy evaluation, a simple
plugin estimator is always more asymptotically sample efficient than the
classical LSTD estimator. For policy optimization, we studied a family of
instances where the convergence of policy gradient to the optimal solution is
guaranteed, and showed that in this setting a simple plugin estimator is
asymptotically at least a factor of state-dimension more efficient than policy gradient, depending
on what specific baseline is used.

This work opens a variety of new directions for future research.
The first is to broaden our results for policy gradient and analyze a 
larger family of instances. 
As mentioned earlier, this would require expanding our current understanding
for under what conditions policy gradient on a finite horizon objective converges to an optimal solution. 
Another interesting direction is to use our framework to analyze
the effect of various baseline estimators in policy gradient.
Designing efficient baseline estimators is still an open problem in RL, 
and using asymptotic analysis to more carefully understand 
the various estimators could be very insightful.
Finally, extending the asymptotic analysis to the online learning setting
may help further our understand of the effects of 
optimistic exploration versus $\varepsilon$-greedy exploration for LQR.

\section*{Acknowledgements}
We thank John Duchi and Daniel Russo, who both independently suggested
studying LQR using asymptotic analysis; this paper is a direct result of their feedback.
We also thank Horia Mania for many helpful discussions regarding policy gradient methods.
Finally, we thank Nicolas Flammarion and Nilesh Tripuraneni for pointers regarding asymptotic analysis of stochastic gradient methods.
ST is supported by a Google PhD fellowship.
BR is generously supported in part by ONR awards N00014-17-1-2191, N00014-17-1-2401, and N00014-18-1-2833, the DARPA Assured Autonomy (FA8750-18-C-0101) and Lagrange (W911NF-16-1-0552) programs, and an Amazon AWS AI Research Award.

\bibliography{paper}

\appendix

\section{Asymptotic Toolbox}

Our analysis relies heavily on computing limiting distributions for the various estimators
we study. A crucial fact we use is that if the matrix $\Lstar$ is stable, then the Markov
chain $\{x_t\}$ given by $x_{t+1} = \Lstar x_t + w_t$ with $w_t \sim \calN(0, \sigma_w^2 I_n)$
is geometrically ergodic. This allows us to apply well known limit theorems
for ergodic Markov chains.

In what follows, we let $\asconv$ denote almost sure convergence and
$\distconv$ denote convergence in distribution.
We also let $\otimes$ denote the standard Kronecker product and
$\otimes_s$ denote the \emph{symmetric} Kronecker product; see e.g.~\citet{schacke13} for a review
of the basic properties of the Kronecker and symmetric Kronecker product which we will use extensively
throughout the sequel.
For a matrix $M$, the notation $\vec(M)$ denotes the vectorized version of $M$ by stacking the columns.
We will also let $\svec(\cdot)$ denote the operator that satisfies
$\ip{\svec(M_1)}{\svec(M_2)} = \ip{M_1}{M_2}$ for all symmetric matrices $M_1, M_2 \in \R^{n \times n}$,
where the first inner product is with respect to $\R^{n(n+1)/2}$ and the second is with respect to $\R^{n \times n}$.
Finally, we let $\mat(\cdot)$ and $\smat(\cdot)$ denote the functional inverses of $\vec(\cdot)$ and $\svec(\cdot)$.
The proofs of the results presented in this section are deferred to Section~\ref{sec:appendix:toolbox_proofs}.

We first state a well-known result that concerns the least-squares estimator of a stable dynamical system.
In the scalar case, this result dates back to \citet{mann43}.
\begin{lem}
\label{lem:ls_asymptotic_dist}
Let $x_{t+1} = \Lstar x_t + w_t$ be a dynamical system with $\Lstar$ stable and $w_t \sim \calN(0, \sigma_w^2 I)$.
Given a trajectory $\{x_t\}_{t=0}^{T}$, let $\Lh(T)$ denote the least-squares estimator of $\Lstar$ with regularization $\lambda \geq 0$:
\begin{align*}
    \widehat{L}(T) = \arg\min_{L \in \R^{n \times n}} \frac{1}{2} \sum_{t=0}^{T-1} \norm{x_{t+1} - L x_t}^2_2 + \frac{\lambda}{2} \norm{L}_F^2 \:.
\end{align*}
Let $P_\infty$ denote the stationary covariance matrix of the process $\{x_t\}_{t=0}^{\infty}$, i.e.
$\Lstar P_\infty \Lstar^\T - P_\infty + \sigma_w^2 I_n = 0$.
We have that $\Lh(T) \asconv \Lstar$ and furthermore:
\begin{align*}
    \sqrt{T} \vec(\Lh(T) - \Lstar) \distconv \calN(0, \sigma_w^2 (P_\infty^{-1} \otimes I_n)) \:.
\end{align*}
\end{lem}

We now consider a slightly altered process where the system is no longer autonomous,
and instead will be driven by white noise.
\begin{lem}
\label{lem:ls_driven_asymptotic_dist}
Let $x_{t+1} = \Astar x_t + \Bstar u_t + w_t$ be a stable dynamical system driven by $u_t \sim \calN(0, \sigma_u^2 I_d)$ and
$w_t \sim \calN(0, \sigma_w^2 I_n)$.
Consider a least-squares estimator $\widehat{\Theta}$ of $\Theta_\star := (\Astar, \Bstar) \in \R^{n \times (n+d)}$
based off of $N$ independent trajectories of length $T$, i.e. given $\{ \{z_t^{(i)} := (x_t^{(i)}, u_t^{(i)})\}_{t=0}^{T} \}_{i=1}^{N}$,
\begin{align*}
  \Thetah(N) = \arg\min_{(A, B) \in \R^{n \times (n+d)}} \frac{1}{2} \sum_{i=1}^{N} \sum_{t=0}^{T-1} \norm{ x_{t+1}^{(i)} - A x_t^{(i)} - B u_t^{(i)} }_2^2 + \frac{\lambda}{2} \norm{\rvectwo{A}{B}}_F^2 \:.
\end{align*}
Let $P_\infty$ denote the stationary covariance of the process $\{x_t\}_{t=0}^{\infty}$, i.e. $P_\infty$ solves
\begin{align*}
    \Astar P_\infty \Astar^\T - P_\infty + \sigma_u^2 \Bstar\Bstar^\T + \sigma_w^2 I_n = 0 \:.
\end{align*}
We have that $\Thetah(N) \asconv \Theta_\star$ and furthermore:
\begin{align*}
  \sqrt{N} \vec(\Thetah(N) - \Theta_\star) \distconv \calN\left(0,  \frac{\sigma_w^2}{T} \bmattwo{P_\infty^{-1}}{0}{0}{(1/\sigma_u^2) I_d} \otimes I_n + o(1/T) \right) \:.
\end{align*}
\end{lem}

Next, we consider the asymptotic distribution of Least-Squares Temporal Difference Learning
for LQR.
\begin{lem}
\label{lem:lstd_asymptotic_distribution}
Let $x_{t+1} = \Astar x_t + \Bstar u_t + w_t$ be a linear system driven by $u_t = K x_t$ and $w_t \sim \calN(0, \sigma_w^2 I_n)$.
Suppose the closed-loop matrix $\Astar + \Bstar K$ is stable.
Let $\nu_\infty$ denote the stationary distribution of the Markov chain $\{x_t\}_{t=0}^{\infty}$.
Define the two matrices $A_\infty, B_\infty$, the mapping $\psi(x)$, and the vector $w_\star$ as
\begin{align*}
    A_\infty &:= \mathop{\E}_{\substack{x \sim \nu_\infty, \\ x' \sim p(\cdot|x, \pi(x))}}[ \phi(x)(\phi(x) - \phi(x'))^\T  ] \:, \\
    B_\infty &:= \mathop{\E}_{\substack{x \sim \nu_\infty, \\ x' \sim p(\cdot|x, \pi(x))}}[ ((\phi(x') - \psi(x))^\T w_\star)^2  \phi(x)\phi(x)^\T ] \:, \\
    \psi(x) &:= \mathop{\E}_{x' \sim p(\cdot|x, \pi(x))}[ \phi(x') ] \:, \\
    w_\star &:= \svec(P_\star) \:.
\end{align*}
Let $\wlstd(T)$ denote the LSTD estimator given by:
\begin{align*}
  \wlstd(T) = \left( \sum_{t=0}^{T-1} \phi(x_t) (\phi(x_t) - \phi(x_{t+1}))^\T \right)^{-1} \left( \sum_{t=0}^{T-1} (c_t - \lambda_t) \phi(x_t) \right) \:.
\end{align*}
Suppose that LSTD is run with the true $\lambda_t = \lambda_\star := \sigma_w^2 \Tr(P_\star)$
and that the matrix $A_\infty$ is invertible.
We have that $\wlstd(T) \asconv w_\star$ and furthermore:
\begin{align*}
    \sqrt{T}(\wlstd(T) - w_\star) \distconv \calN(0, A_\infty^{-1} B_\infty A_\infty^{-\T}) \:.
\end{align*}
\end{lem}

As a corollary to Lemma~\ref{lem:lstd_asymptotic_distribution}, we work out the formulas for $A_\infty$ and
$B_\infty$ and a useful lower bound.
\begin{cor}
\label{cor:a_infty_b_infty}
In the setting of Lemma~\ref{lem:lstd_asymptotic_distribution}, with $\Lstar = \Astar + \Bstar K$, we have
that the matrix $A_\infty$ is invertible, and:
\begin{align*}
    A_\infty &= (P_\infty \otimes_s P_\infty) - (P_\infty \Lstar^\T \otimes_s P_\infty \Lstar^\T) \:, \\
    B_\infty &= ( \sigma_w^2\ip{P_\infty}{\Lstar^\T P_\star^2 \Lstar} + 2 \sigma_w^4 \norm{P_\star }_F^2  ) (2 (P_\infty \otimes_s P_\infty) +  \svec(P_\infty) \svec(P_\infty)^\T) \\
    &\qquad + 2 \sigma_w^2 (\svec(P_\infty) \svec(P_\infty \Lstar^\T P_\star^2 \Lstar P_\infty)^\T + \svec(P_\infty \Lstar^\T P_\star^2 \Lstar P_\infty)\svec(P_\infty)^\T ) \\
    &\qquad + 8 \sigma_w^2 (P_\infty \Lstar^\T P_\star^2 \Lstar P_\infty \otimes_s P_\infty) \:.
\end{align*}
Furthermore, we can lower bound the matrix $A_\infty^{-1} B_\infty A_\infty^{-\T}$ by:
\begin{align}
  A_\infty^{-1} B_\infty A_{\infty}^{-\T} &\succeq 8 \sigma_w^2 \ip{P_\infty}{\Lstar^\T P_\star^2 \Lstar} (I - \Lstar^\T \otimes_s \Lstar^\T)^{-1} (P_\infty^{-1} \otimes_s P_\infty^{-1})(I - \Lstar^\T \otimes_s \Lstar^\T)^{-\T} \nonumber \\
  &\qquad + 16 \sigma_w^2(I - \Lstar^\T \otimes_s \Lstar^\T)^{-1} (\Lstar^\T P_\star^2 \Lstar \otimes_s P_\infty^{-1})(I - \Lstar^\T \otimes_s \Lstar^\T)^{-\T} \label{eq:lstd_cov_lb} \:.
\end{align}
\end{cor}

Next, we state a standard lemma which we will use to convert convergence in distribution
guarantees to guarantees regarding the convergence of risk.
\begin{lem}
\label{lem:lp_convergence}
Suppose that $\{X_n\}$ is a sequence of random vectors and $X_n \distconv X$.
Suppose that $f$ is a non-negative continuous real-valued function such that $\E[f(X)] < \infty$.
We have that:
\begin{align*}
    \liminf_{n \to \infty} \E[ f(X_n) ] \geq \E[ f(X) ] \:.
\end{align*}
If additionally we have $\sup_{n \geq 1} \E[ f(X_n)^{1+\varepsilon} ] < \infty$ holds for some $\varepsilon > 0$, then
the limit $\lim_{n \to \infty} \E[ f(X_n) ]$ exists and
\begin{align*}
    \lim_{n \to \infty} \E[ f(X_n) ] = \E[ f(X) ] \:.
\end{align*}
\end{lem}
\begin{proof}
Both facts are standard consequences of weak convergence of
probability measures; see e.g. Chapter 5 of~\citet{billingsley95} for more details.
\end{proof}

The next claim uniformly controls the $p$-th moments of the regularized least-squares estimate
when $T$ is large enough. This technical result will allow us to invoke Lemma~\ref{lem:lp_convergence}
to obtain convergence in $L^p$.
\begin{lem}
\label{lem:moment_control_ls}
Let $x_{t+1} = \Lstar x_t + w_t$ with $w_t \sim \calN(0, \sigma_w^2 I_n)$ and $\Lstar$ stable.
Fix a regularization parameter $\lambda > 0$ and let $\Lh(T)$ denote the LS estimator:
\begin{align*}
    \Lh(T) = \arg\min_{L \in \R^{n \times n}} \frac{1}{2} \sum_{t=0}^{T-1} \norm{x_{t+1} - L x_t}^2_2 + \frac{\lambda}{2} \norm{L}_F^2 \:.
\end{align*}
Fix a finite $p \geq 1$.
Let $C_{\Lstar,\lambda,n}$ and $C_{\Lstar, \lambda, n, p}$ denote constants that depend only on $\Lstar, \lambda, n$ (resp. $\Lstar, \lambda, n, p$)
  and not on $T, \delta$. Fix a $\delta \in (0, 1)$. With probability at least $1-\delta$, as long as $T \geq C_{\Lstar,\lambda, n}\log(1/\delta)$ we have:
\begin{align*}
  \norm{ \Lh(T) - \Lstar } \leq C'_{\Lstar,\lambda,n} \sqrt{\frac{\log(1/\delta)}{T}} \:.
\end{align*}
Furthermore, as long as $T \geq C_{\Lstar,\lambda,n,p}$, then:
\begin{align*}
    \E[ \norm{\Lh(T) - \Lstar}^{p} ] \leq C'_{\Lstar,\lambda,n,p} \frac{1}{T^{p/2}} \:.
\end{align*}
\end{lem}

The next result is
the analogue of Lemma~\ref{lem:moment_control_ls}
for the non-autonomous system driven by white noise.
\begin{lem}
\label{lem:moment_control_driven_ls}
Let $x_{t+1} = \Astar x_t + \Bstar u_t + w_t$ with $w_t \sim \calN(0, \sigma_w^2 I_n)$,
$u_t \sim \calN(0, \sigma_u^2 I_d)$, and $\Astar$ stable.
Fix a regularization parameter $\lambda > 0$ and let $\Thetah(N)$ denote the LS estimator:
\begin{align*}
  \Thetah(N) = \arg\min_{(A, B) \in \R^{n \times (n + d)}} \frac{1}{2} \sum_{i=1}^{N} \sum_{t=0}^{T-1} \norm{x_{t+1}^{(i)} - A x_t^{(i)} - B u_t^{(i)}}^2_2 + \frac{\lambda}{2} \norm{\rvectwo{A}{B}}_F^2 \:.
\end{align*}
Fix a finite $p \geq 1$.
Let $C_{\Theta_\star, T, \lambda, n, d}$ and $C_{\Theta_\star, T, \lambda, n, d, p}$ denote constants that depend only on $\Theta_\star, T, \lambda, n, d$ (resp. $\Theta_\star, T, \lambda, n, d, p$)
and not on $N, \delta$. Fix a $\delta \in (0, 1)$. With probability at least $1-\delta$, as long as $N \geq C_{\Theta_\star,T,\lambda, n, d}\log(1/\delta)$ we have:
\begin{align*}
    \norm{ \Thetah(N) - \Theta_\star } \leq C'_{\Theta_\star,T,\lambda,n,d} \sqrt{\frac{\log(1/\delta)}{N}} \:.
\end{align*}
Furthermore, as long as $N \geq C_{\Theta_\star,T,\lambda,n,d,p}$, then:
\begin{align*}
    \E[ \norm{\Thetah(N) - \Theta_\star}^{p} ] \leq C'_{\Theta_\star,T,\lambda,n,d,p} \frac{1}{N^{p/2}} \:.
\end{align*}
\end{lem}
\begin{proof}
The proof is nearly identical to that of Lemma~\ref{lem:moment_control_ls},
except we use the concentration result of Proposition 1.1 from \citet{dean17}
instead of Theorem 2.4 of \citet{simchowitz18}
to establish concentration over multiple independent rollouts.
We omit the details as they very closely mimic that of Lemma~\ref{lem:moment_control_ls}.

We note that in doing this we obtain a sub-optimal dependence on the horizon
length $T$. This can be remedied by a more careful argument combining
the concentration along each trajectory from Simchowitz et al. with the
concentration across independent trajectories from Dean et al.
However, as in our limit theorems only $N$ the rollout length is
being sent to infinity (e.g. $T$ is considered a constant), a sub-optimal bound in $T$ will suffice
for our purpose.
\end{proof}

Our final asymptotic result deals with the performance of stochastic gradient descent (SGD)
with projection. This will be our key ingredient in analyzing policy gradient (Algorithm~\ref{alg:model_free_policy_opt}).
While the asymptotic performance of SGD (and more generally stochastic approximation)
is well-established (see e.g.~\citet{kushner03}),
we consider a slight modification where the iterates are projected back into a compact convex set
at every iteration.
As long as the optimal solution is not on the boundary of the projection set, then
one intuitively does not expect the asymptotic distribution to be affected by this projection,
since eventually as SGD converges towards the optimal solution the projection step will effectively be inactive.
Our result here makes this intuition rigorous. It follows by combining
the asymptotic analysis of \citet{toulis17}
with the high probability bounds for SGD from \citet{rakhlin12}.

To state the result, we need a few definitions.
First, we say a differentiable function $F : \R^d \rightarrow \R$ satisfies
\emph{restricted strong convexity} (RSC) on a compact convex set $\Theta \subseteq \R^d$
if it has a unique minimizer $\theta_\star \in \mathrm{int}(\Theta)$ and for some $m > 0$,
we have
$\ip{\nabla F(\theta)}{\theta - \theta_\star} \geq m \norm{\theta - \theta_\star}_2^2$
for all $\theta \in \Theta$.
We denote this by $\mathsf{RSC}(m, \Theta)$.

\begin{lem}
\label{lem:sgd_asymptotics}
Let $F \in \calC^3(\Theta)$
and suppose $F$ satisfies $\mathsf{RSC}(m, \Theta)$.
Let $\theta_\star \in \Theta$ denote the unique minimizer of $F$ in $\Theta$.
Suppose we have a stochastic
gradient oracle $g(\theta; \xi)$ such that $g$ is continuous in both $\theta,\xi$ and
$\nabla F(\theta) = \E_{\xi}[g(\theta; \xi)]$ for some distribution over $\xi$.
Suppose that for some $G_1, G_2, L > 0$, for all $p \in [1, 4]$ and $\delta \in (0, 1)$, we have that
\begin{align}
    \sup_{\theta \in \Theta} \E_{\xi}[ \norm{g(\theta; \xi)}_2^p ] &\leq G_1^p \:, \label{eq:sgd_grad_momoment_bound} \\
    \Pr_{\xi}\left(\sup_{\theta \in \Theta} \norm{ g(\theta; \xi) }_2 > G_2 \polylog(1/\delta)\right) &\leq \delta \:, \label{eq:sgd_grad_conc} \\
    \E_{\xi}[ \norm{g(\theta; \xi) - g(\theta_\star; \xi)}_2^2 ] &\leq L \norm{\theta - \theta_\star}_2^2 \:\: \forall \theta \in \Theta \label{eq:sgd_grad_lip} \:.
\end{align}
Given an sequence $\{\xi_t\}_{t=1}^{\infty}$ drawn i.i.d.\ from the law of $\xi$, consider the sequence of iterates $\{\theta_t\}_{t=1}^{\infty}$
starting with $\theta_1 \in \Theta$ and defined as:
\begin{align*}
    \theta_{t+1} = \mathsf{Proj}_{\Theta}( \theta_t - \alpha_t g(\theta_t; \xi_t) ) \:, \:\: \alpha_t = \frac{1}{m t} \:.
\end{align*}
We have that:
\begin{align}
    \lim_{T \to \infty} mT \cdot \Var(\theta_T) = \Xi \:, \label{eq:sgd_asymptotic_var}
\end{align}
where $\Xi = \lyap(\frac{m}{2}I_d - \nabla^2 F(\theta_\star), \E_{\xi}[g(\theta_\star;\xi) g(\theta_\star;\xi)^\T] )$ solves the continuous-time Lyapunov equation:
\begin{align}
    \left(\frac{m}{2} I_d - \nabla^2 F(\theta_\star) \right) \Xi + \Xi \left(\frac{m}{2} I_d - \nabla^2 F(\theta_\star) \right) + \E_{\xi}[ g(\theta_\star; \xi)g(\theta_\star; \xi)^\T ] = 0 \:. \label{eq:cts_lyap}
\end{align}
We also have that for any $G \in \calC^3(\Theta)$
with $\nabla G(\theta_\star) = 0$ and $\nabla^2 G(\theta_\star) \succ 0$,
\begin{align}
    \liminf_{T \to \infty} T \cdot \E[G(\theta_T) - G(\theta_\star)] \geq \frac{1}{2m} \Tr( \nabla^2 G(\theta_\star) \cdot \Xi ) \:. \label{eq:sgd_comparison_bound}
\end{align}
\end{lem}
We defer the proof of this lemma to Section~\ref{sec:proof_sgd_asymptotics} of the Appendix.
We quickly comment on how the last inequality can be used.
Taking trace of both sides from Equation~\ref{eq:cts_lyap}, we obtain:
\begin{align*}
    \Tr(\Xi \cdot (\nabla^2 F(\theta_\star) - \frac{m}{2} I_d)) = \frac{1}{2} \E_{\xi}[ \norm{g(\theta_\star; \xi)}_2^2 ] \:.
\end{align*}
We now upper bound the LHS as:
\begin{align*}
  \Tr( \Xi \cdot (\nabla^2 F(\theta_\star) - \frac{m}{2} I_d) ) &= \Tr( \Xi \cdot \nabla^2 G(\theta_\star)^{1/2} \cdot  \nabla^2 G(\theta_\star)^{-1/2} (\nabla^2 F(\theta_\star) - \frac{m}{2} I_d) \nabla^2 G(\theta_\star)^{-1/2} \cdot \nabla^2 G(\theta_\star)^{1/2}  ) \\
    &\leq \Tr( \Xi \cdot \nabla^2 G(\theta_\star) ) \lambda_{\max}( \nabla^2 G(\theta_\star)^{-1/2}(\nabla^2 F(\theta_\star) - \frac{m}{2} I_d) \nabla^2 G(\theta_\star)^{-1/2} ) \\
    &= \Tr( \Xi \cdot \nabla^2 G(\theta_\star) ) \lambda_{\max}( \nabla^2 G(\theta_\star)^{-1}(\nabla^2 F(\theta_\star) - \frac{m}{2} I_d)) \:.
\end{align*}
Combining the last two equations we obtain that:
\begin{align}
  \liminf_{T \to \infty} T \cdot \E[ G(\theta_T) - G(\theta_\star)] &\geq \frac{1}{2m} \Tr( \Xi \cdot \nabla^2 G(\theta_\star) ) \nonumber \\
  &\geq \frac{1}{4m \lambda_{\max}( \nabla^2 G(\theta_\star)^{-1}( \nabla^2 F(\theta_\star) - \frac{m}{2} I_d)) } \E_\xi[ \norm{g(\theta_\star; \xi)}_2^2 ] \:. \label{eq:cost_lower_bound_sgd}
\end{align}
We will use this last estimate in our analysis.

\section{Analysis of Policy Evalution Methods}

In this section, recall that $Q, R, K$ are fixed, and furthermore define $M := Q + K^\T R K$.

\subsection{Proof of Theorem~\ref{thm:policy_eval_plugin_risk}}
\label{sec:proofs:policy_eval_plugin}

The strategy is as follows. Recall that Lemma~\ref{lem:ls_asymptotic_dist} gives us the asymptotic
distribution of the (regularized) least-squares estimator $\Lh(T)$ of the true closed-loop matrix $\Lstar$.
For a stable matrix $L$, let $P(L) = \dlyap(L, M)$. Since the map $L \mapsto P(L)$ is differentiable,
using the delta method we can recover the asymptotic distribution of $\sqrt{T}\svec(P(\Lh(T)) - P_\star)$.
Upper bounding the trace of the covariance matrix for this asymptotic distribution then yields
Theorem~\ref{thm:policy_eval_plugin_risk}.

Let $[DP(L)]$ denote the Fr{\'e}chet derivative of the map $P(\cdot)$ evaluated at $L$, and let
$[DP(L)](X)$ denote the action of the linear operator $[DP(L)]$ on $X$. By a straightforward application of the
implicit function theorem, we have that:
\begin{align*}
    [DP(\Lstar)](X) = \dlyap(\Lstar, X^\T P_\star \Lstar + \Lstar^\T P_\star X) \:.
\end{align*}

Before we proceed, we introduce some notation surrounding Kronecker products.
Let $\Gamma$ denote the matrix such that $(A \otimes_s B) = \frac{1}{2} \Gamma^\T (A \otimes B + B \otimes A) \Gamma$
for any square matrices $A,B$.
It is a fact that $\Gamma \vec(S) = \svec(S)$ for any symmetric matrix $S$.
Also let $\Pi$ be the orthonormal matrix such that $\Pi\vec(X) = \vec(X^\T)$ for all square matrices $X$.
It is not hard to verify that $\Pi^\T (A \otimes B) \Pi = (B \otimes A)$, a fact we will use later.
With this notation, we proceed as follows:
\begin{align*}
  \svec([DP(\Lstar)](X)) &= (I - \Lstar^\T \otimes_s \Lstar^\T)^{-1} \svec(X^\T P_\star \Lstar + \Lstar^\T P_\star X) \\
  &= (I - \Lstar^\T \otimes_s \Lstar^\T)^{-1} \Gamma \vec(X^\T P_\star \Lstar + \Lstar^\T P_\star X) \\
  &= (I - \Lstar^\T \otimes_s \Lstar^\T)^{-1} \Gamma ((\Lstar^\T P_\star \otimes I_n) \Pi + (I_n \otimes \Lstar^\T P_\star))\vec(X) \:.
\end{align*}
Applying Lemma~\ref{lem:ls_asymptotic_dist} in conjunction with the delta method,
we obtain:
\begin{align*}
  \sqrt{T} \svec(P(\Lh(T)) - P_\star) \distconv \calN(0, \sigma_w^2 (I - \Lstar^\T \otimes_s \Lstar^\T)^{-1} V (I - \Lstar^\T \otimes_s \Lstar^\T)^{-\T}  ) \:,
\end{align*}
where,
\begin{align*}
  V &:= \Gamma [((\Lstar^\T P_\star \otimes I_n) \Pi + (I_n \otimes \Lstar^\T P_\star)) (P_\infty^{-1} \otimes I_n) ((\Lstar^\T P_\star \otimes I_n) \Pi + (I_n \otimes \Lstar^\T P_\star))^\T] \Gamma^\T \\
  &\stackrel{(a)}{\preceq} 2 \Gamma[  (\Lstar^\T P_\star \otimes I_n) \Pi (P_\infty^{-1} \otimes I_n) \Pi^\T (P_\star \Lstar \otimes I_n) + (I_n \otimes \Lstar^\T P_\star)(P_\infty^{-1} \otimes I_n) (I_n \otimes P_\star \Lstar)]\Gamma^\T \\
  &= 2 \Gamma[(\Lstar^\T P_\star \otimes I_n) (I_n \otimes P_\infty^{-1})  (P_\star \Lstar \otimes I_n) + (I_n \otimes \Lstar^\T P_\star)(P_\infty^{-1} \otimes I_n) (I_n \otimes P_\star \Lstar)]\Gamma^\T \\
  &= 2 \Gamma[(\Lstar^\T P_\star^2 \Lstar \otimes P_\infty^{-1}) + (P_\infty^{-1} \otimes \Lstar^\T P_\star^2 \Lstar)]\Gamma^\T \\
  &= 4 (\Lstar^\T P_\star^2 \Lstar \otimes_s P_\infty^{-1}) \:.
\end{align*}
In (a), we used the inequality
for any matrices $X, Y$ and positive definite matrices $F, G$,
(see e.g. Chapter 3, page 94 of \citet{zhang05}):
\begin{align*}
  (X+Y) (F+G)^{-1} (X+Y)^\T \preceq X F^{-1} X^\T + Y G^{-1} Y^\T \:.
\end{align*}
Suppose that the sequence $\{ \norm{Z_T}_F^2 \}$ is uniformly integrable,
where $Z_T := \sqrt{T} \svec(P(\Lh(T)) - P_\star)$.
Then:
\begin{align*}
  \lim_{T \to \infty} T \cdot \E[ \norm{P(\Lh(T)) - P_\star}_F^2 ] \leq 4 \Tr( (I - \Lstar^\T \otimes_s \Lstar^\T)^{-1}(\Lstar^\T P_\star^2 \Lstar \otimes_s \sigma_w^2P_\infty^{-1}) (I - \Lstar^\T \otimes_s \Lstar^\T)^{-\T}) \:,
\end{align*}
which is the desired bound on the asymptotic risk.

We now show that the sequence $\{ \norm{Z_T}_F^2 \}$ is uniformly integrable.
To do this, we need a simple matrix stability perturbation bound.
\begin{lem}
\label{lem:simple_perturbation}
Let $A$ be a stable matrix that satisfies $\norm{A^k} \leq C \rho^k$ for all $k \geq 0$ for some $C > 0$ and
$\rho \in (0, 1)$. Fix a $\gamma \in (\rho, 1)$.
Suppose that $\Delta$ is a perturbation that satisfies:
\begin{align*}
  \norm{\Delta} \leq \frac{\gamma - \rho}{C} \:.
\end{align*}
Then we have that (a) $A+\Delta$ is a stable matrix with $\rho(A+\Delta) \leq \gamma$
and (b) $\norm{(A+\Delta)^k} \leq C \gamma^k$ for all $k \geq 0$.
\end{lem}
\begin{proof}
We start by proving (b).
Fix an integer $k \geq 1$.
Consider the expansion of $(A+\Delta)^k$ into $2^k$ terms.
Label all these terms as $T_{i,j}$ for $i=0, ..., k$ and $j=1, ..., {k \choose i}$
where $i$ denotes the degree of $\Delta$ in the term (hence there are ${k \choose i}$ terms with a degree of $i$ for $\Delta$).
Using the fact that $\norm{A^k} \leq C \rho^k$ for all $k \geq 0$, we can
bound $\norm{T_{i,j}} \leq C^{i+1} \rho^{k-i} \norm{\Delta}^i$.
Hence by triangle inequality:
\begin{align*}
  \norm{(A+\Delta)^k} &\leq \sum_{i=0}^{k} \sum_{j} \norm{T_{i,j}} \\
  &\leq \sum_{i=0}^{k} {k \choose i} C^{i+1} \rho^{k-i} \norm{\Delta}^i \\
  &= C \sum_{i=0}^{k} {k \choose i} (C \norm{\Delta})^i \rho^{k-i} \\
  &= C (C\norm{\Delta} + \rho)^k \\
  &\leq C \gamma^k \:,
\end{align*}
where the last inequality uses the assumption $\norm{\Delta} \leq \frac{\gamma-\rho}{C}$.
This gives the claim (b).

To derive the claim (a),
we use the inequality that
$\rho(A + \Delta) \leq \norm{ (A+\Delta)^k }^{1/k} \leq C^{1/k} \gamma$ for any $k \geq 1$.
Since this holds for any $k \geq 1$, we can take the infimum over all $k \geq 1$ on the
RHS, which yields the desired claim.
\end{proof}

\newcommand{\calEbdd}{\calE_{\mathsf{Bdd}}}
\newcommand{\calEalg}{\calE_{\mathsf{Alg}}}
\newcommand{\calGalg}{\calG_{\mathsf{Alg}}}

Fix a finite $p \geq 1$.
Since $\Lstar$ is stable
and $\zeta \in (\rho(\Lstar), 1)$, there exists a $C_\star$ such that
$\norm{\Lstar^k} \leq C_\star \zeta^k$ for all $k \geq 0$.
For the rest of the proof, $O(\cdot), \Omega(\cdot)$ will hide constants that depend on $\Lstar, C_\star, n, p, \lambda,\zeta,\psi$, but not on $T$.
Set $\delta_T = O(1/T^{p/2})$ and let $T$ be large enough so that
there exists an event
$\calEbdd$ promised by Lemma~\ref{lem:moment_control_ls}
such that $\Pr(\calEbdd) \geq 1 - \delta_T$ and
on $\calEbdd$ we have $\norm{\Lh(T) - \Lstar} \leq O(\sqrt{\log(1/\delta_T)/T})$.
Let $T$ also be large enough so that on $\calEbdd$, we have
$\norm{\Lh(T) - \Lstar} \leq \min((\gamma-\rho_\star)/C_\star, \psi - \norm{\Lstar})$.
With this setting, we have that on $\calEbdd$,
for any $\alpha \in (0,1)$,
\begin{align*}
  \tilde{L}(\alpha) := \alpha \Lh(T) + (1-\alpha) \Lstar \in \left\{ L \in \R^{n \times n} : \rho(L) \leq \zeta \:, \:\: \norm{L} \leq \min\left(\norm{\Lstar} + \frac{\gamma-\rho_\star}{C_\star}, \psi\right)  \right\} =: \calG \:.
\end{align*}
Therefore on $\calEbdd$, for some $\alpha \in (0, 1)$,
\begin{align*}
  \norm{P(\Lh(T)) - P_\star} = \norm{[DP(\tilde{L}(\alpha))]( \Lh(T) - \Lstar )} \leq \sup_{\tilde{L} \in \calG} \norm{[DP(\tilde{L})]} \norm{\Lh(T) - \Lstar} := S \norm{\Lh(T) - \Lstar} \:.
\end{align*}
Here the norm $\norm{[H]} := \sup_{\norm{X} \leq 1} \norm{[H](X)}$.
We have that $S$ is finite since $\calG$ is a compact set.
Next, define the set $\calGalg$ as:
\begin{align*}
  \calGalg := \{ L \in \R^{n \times n} : \rho(L) \leq \zeta \:, \:\: \norm{L} \leq \psi \} \:,
\end{align*}
and define the event $\calEalg$ as $\calEalg := \{ \Lh(T) \in \calGalg \}$.
Consider the decomposition:
\begin{align*}
  \E[ \norm{\Phplugin(T) - P_\star}^p ] &= \E[ \norm{\Phplugin(T) - P_\star}^p \ind_{\calEbdd}  ] + \E[ \norm{\Phplugin(T) - P_\star}^p \ind_{\calEbdd^c} ] \\
  &\leq \E[ \norm{\Phplugin(T) - P_\star}^p \ind_{\calEbdd}  ] + \E[ \norm{\Phplugin(T) - P_\star}^p \ind_{\calEbdd^c \cap \calEalg} ] \\
  &\qquad+ \E[ \norm{\Phplugin(T) - P_\star}^p \ind_{\calEbdd^c \cap \calEalg^c} ] \:.
\end{align*}
In what follows we will assume that $T$ is sufficiently large.

\paragraph{On $\calEbdd$.}
On this event, since we have $\calEbdd \subseteq \calEalg$,
we can bound by Lemma~\ref{lem:moment_control_ls}:
\begin{align*}
  \E[ \norm{\Phplugin(T) - P_\star}^p \ind_{\calEbdd}  ] &= \E[ \norm{\Phplugin(T) - P_\star}^p \ind_{\calEbdd \cap \calEalg} ]
  \leq S^p \E[ \norm{\Lh(T) - \Lstar}^p ] \leq O(1/T^{p/2}) \:.
\end{align*}

\paragraph{On $\calEbdd^c \cap \calEalg$.}
On this event, we use the fact that $\calGalg$ is compact to bound:
\begin{align*}
  \E[ \norm{\Phplugin(T) - P_\star}^p \ind_{\calEbdd^c \cap \calEalg} ] &\leq \sup_{\Lh \in \calGalg} \norm{\dlyap(\Lh, Q + K^\T R K) - P_\star}^p \Pr( \calEbdd^c \cap \calEalg)  \\
  &\leq \sup_{\Lh \in \calGalg} \norm{\dlyap(\Lh, Q + K^\T R K) - P_\star}^p \delta_T \\
  &\leq O(1/T^{p/2}) \:.
\end{align*}

\paragraph{On $\calEbdd^c \cap \calEalg^c$.}
On this event, we simply have:
\begin{align*}
  \E[ \norm{\Phplugin(T) - P_\star}^p \ind_{\calEbdd^c \cap \calEalg^c} ] = \norm{P_\star}^p \Pr(\calEbdd^c \cap \calEalg^c) \leq \norm{P_\star}^p \delta_T \leq O(1/T^{p/2}) \:.
\end{align*}

\paragraph{Putting it together.}
Combining these bounds we obtain that
$\E[ \norm{\Phplugin(T) - P_\star}^p ] \leq O(1/T^{p/2})$.
Recall that $Z_T = \svec(P(\Lh(T)) - P_\star)$.
We have that for
any finite $\gamma > 0$ and $T \geq \Omega(1)$:
\begin{align*}
    \E[ \norm{Z_T}_F^{2+\gamma} ] &= T^{(2+\gamma)/2} \E[ \norm{P(\Lh(T)) - P_\star}_F^{2+\gamma} ]  \\
    &\leq n^{(2+\gamma)/2} T^{(2+\gamma)/2} \E[ \norm{P(\Lh(T)) - P_\star}^{2+\gamma} ] \\
    &\leq n^{(2+\gamma)/2} T^{(2+\gamma)/2} O(1/T^{(2+\gamma)/2}) \\
    &\leq n^{(2+\gamma)/2} O(1) \:.
\end{align*}
On the other hand, when $T \leq O(1)$ it is easy to see
that $\E[ \norm{Z_T}_F^{2+\gamma} ]$ is finite.
Hence we have $\sup_{T \geq 1} \E[ \norm{Z_T}_F^{2+\gamma} ] < \infty$ which shows
the desired uniformly integrable condition.
This concludes the proof of Theorem~\ref{thm:policy_eval_plugin_risk}.

\subsection{Proof of Theorem~\ref{thm:policy_eval_lstd_risk}}

\label{sec:proofs:policy_eval_lstd}

Lemma~\ref{lem:lstd_asymptotic_distribution} (specifically \eqref{eq:lstd_cov_lb}) combined with
Lemma~\ref{lem:lp_convergence} tells us that:
\begin{align*}
  &\liminf_{T \to \infty} T \cdot \E[ \norm{ \Phlstd(T) - P_\star }_F^2 ] \geq \Tr( A_\infty^{-1} B_\infty A_\infty^{-\T}  ) \\
    &\qquad\geq 8 \sigma_w^2 \Tr(\ip{P_\infty}{\Lstar^\T P_\star^2 \Lstar} (I - \Lstar^\T \otimes_s \Lstar^\T)^{-1} (P_\infty^{-1} \otimes_s P_\infty^{-1})(I - \Lstar^\T \otimes_s \Lstar^\T)^{-\T}) \nonumber \\
  &\qquad\qquad + 16 \sigma_w^2 \Tr( (I - \Lstar^\T \otimes_s \Lstar^\T)^{-1} (\Lstar^\T P_\star^2 \Lstar \otimes_s P_\infty^{-1})(I - \Lstar^\T \otimes_s \Lstar^\T)^{-\T}) \:.
\end{align*}
The claim now follows by using the risk bound from Theorem~\ref{thm:policy_eval_plugin_risk}.

\subsection{Proof of Theorem~\ref{thm:policy_eval_lower_bound}}

Let $E_1, ..., E_N$ be $d$-dimensional subspaces of $\R^n$ with $d \leq n/2$ such that
$\norm{P_{E_i} - P_{E_j}}_F \gtrsim \sqrt{d}$.
By Proposition 8 of \citet{pajor98},
we can take $N \geq e^{n(n-d)}$.
Now consider instances $A_i$ with $A_i = \tau P_{E_i} + \gamma I_n$ for a $\tau,\gamma \in (0, 1)$ to be determined.
We will set $\tau + \gamma = \rho$ so that each $A_i$ is contractive (i.e. $\norm{A_i} < 1$) and hence stable.
This means implicitly that we will require $\tau < \rho$.
Let $\Pr_i$ denote the distribution over $(x_1, ..., x_T)$ induced by instance $A_i$.
We have that:
\begin{align*}
    \KL(\Pr_i, \Pr_j) &= \sum_{t=1}^{T} \E_{x_t \sim \Pr_i}[ \KL(\calN(A_i x_t, \sigma^2 I), \calN(A_j x_t, \sigma^2 I)) ] \\
    &= \frac{1}{2\sigma^2} \sum_{t=1}^{T} \E_{x_t \sim \Pr_i}[ \norm{(A_i - A_j) x_t}^2_2 ] \\
    &\leq \frac{\norm{A_i - A_j}^2}{2\sigma^2} \sum_{t=1}^{T} \Tr( \E_{x_t \sim \Pr_i}[ x_tx_t^\T ] ) \\
    &\leq \frac{\tau^2}{\sigma^2} T \Tr(P_\infty) \\
    &= \tau^2 T \left( \frac{d}{1-\rho^2} + \frac{n-d}{1-\gamma^2} \right) \\
    &\leq \tau^2 T \frac{n}{1-\rho^2} \:.
\end{align*}
Now if we choose $n(n-d) \geq 4 \log{2}$ and $T \gtrsim n (1-\rho^2)/\rho^2$, we can set $\tau^2 \asymp \frac{n(1-\rho^2)}{T}$ and
obtain that $\frac{I(V;X) + \log{2}}{\log{\abs{V}}} \leq 1/2$.

On the other hand, let $P_i = \dlyap(A_i, I_n)$.
We have that for any integer $k \geq 0$:
\begin{align*}
  (\tau P_{E_i} + \gamma I_n)^k - (\tau P_{E_j} + \gamma I_n)^k &= \sum_{\ell=0}^{k} {k \choose \ell} \gamma^{k-\ell} \tau^\ell (P_{E_i}^\ell - P_{E_j}^\ell) \\
  &= k \gamma^{k-1} \tau (P_{E_i} - P_{E_j}) + \sum_{\ell=2}^{k}{k \choose \ell} \gamma^{k-\ell} \tau^\ell (P_{E_i}^\ell - P_{E_j}^\ell) \:.
\end{align*}
Hence,
\begin{align*}
  P_i - P_j &= \sum_{k=1}^{\infty} ( (A_i^k)^\T A_i^k - (A_j^k)^\T A_j^k ) \\
  &= \sum_{k=1}^{\infty} (A_i^{2k} - A_j^{2k}) \\
  &= \left(\sum_{k=1}^{\infty} 2k \gamma^{2k-1} \tau  + \sum_{k=2}^{\infty} \sum_{\ell=2}^{2k} {2k \choose \ell} \gamma^{2k-\ell} \tau^\ell \right) (P_{E_i} - P_{E_j}) \\
  &= \left( \frac{2\gamma\tau}{(1-\gamma^2)^2} + \sum_{k=2}^{\infty} \sum_{\ell=2}^{k} {k \choose \ell} \gamma^{k-\ell} \tau^\ell \right) (P_{E_i} - P_{E_j}) \:.
\end{align*}
Therefore,
\begin{align*}
  \norm{P_i - P_j}_F &\geq \frac{2\gamma\tau}{(1-\gamma^2)^2} \norm{P_{E_i} - P_{E_j}}_F \gtrsim \frac{\gamma \tau}{(1-\gamma^2)^2} \sqrt{d} \:.
\end{align*}
The claim now follows by Fano's inequality and setting $d = n/4$.

\section{Analysis of Policy Optimization Methods}

\subsection{Preliminary Calculations}
\label{sec:pg:calcs}

Given $(\Astar, \Bstar)$ with $\mathrm{range}(\Astar) \subseteq \mathrm{range}(\Bstar)$ and $\rank(\Bstar) = d$, let
$J_{\Sigma}(K)$ for a $K \in \R^{d \times n}$ denote the following cost:
\begin{align*}
  J_{\Sigma}(K) := \E\left[ \sum_{t=1}^{T} \norm{x_t}^2_2 \right]  \:, \:\: x_{t+1} = \Astar x_t + \Bstar u_t + w_t \:, \:\: u_t = K x_t \:, \:\: w_t \sim \calN(0, \Sigma) \:.
\end{align*}
Here we assume $T \geq 2$ and $\Sigma$ is positive definite.
We write $J(K) = J_{\sigma_w^2 I_n}(K)$ as shorthand.
Under this feedback law, we have $x_t \sim \calN(0, \sum_{\ell=0}^{t-1} L(K)^\ell \Sigma (L(K)^\ell)^\T )$
with $L(K) := \Astar + \Bstar K$.
Letting $L$ be shorthand for $L(K)$, the cost can be written as:
\begin{align*}
  J_{\Sigma}(K) &= \sum_{t=1}^{T} \sum_{\ell=0}^{t-1} \Tr( L^\ell \Sigma (L^\ell)^\T )
  = T \Tr(\Sigma) + \sum_{t=1}^{T} \sum_{\ell=1}^{t-1} \Tr( L^\ell \Sigma (L^\ell)^\T ) \:.
\end{align*}
Let $\Kstar$ denote the minimizer of $J_{\Sigma}(K)$; under our assumptions we have that $\Kstar = -\Bstar^{\dag} \Astar$.
Furthermore, because of the range condition we can write $\Astar = \Bstar\Bstar^{\dag} \Astar$.
Therefore, $L(K) = \Bstar (\Bstar^\dag \Astar + K)$.
While the function $J_{\Sigma}(K)$ is not convex, it has many nice properties.
First, $J_{\Sigma}(K)$ satisfies a \emph{quadratic growth condition}:
\begin{align}
  J_{\Sigma}(K) - J_{\Sigma}(\Kstar) &\geq (T-1) \Tr( L \Sigma L^\T ) \nonumber \\
  &= (T-1) \Tr( \Bstar (\Bstar^\dag \Astar + K) \Sigma (\Bstar^\dag \Astar + K)^\T \Bstar^\T ) \nonumber \\
  &= (T-1) \vec(\Bstar^\dag \Astar + K)^\T ( \Sigma \otimes \Bstar^\T \Bstar ) \vec(\Bstar^\dag \Astar + K) \nonumber \\
  &\geq (T-1) \lambda_{\min}(\Sigma) \sigma_{\min}(\Bstar)^2 \norm{K - \Kstar}_F^2 \label{eq:pg:quadratic_growth} \:.
\end{align}

Next, we will see $J_{\Sigma}(K)$ satisfies restricted strong convexity.
To do this, we first compute the gradient $\nabla J_{\Sigma}(K)$.
Consider the function $M \mapsto M^\ell$ for any integer $\ell \geq 2$.
We have that the derivatives are:
\begin{align*}
  [D M^\ell](\Delta) &= \sum_{k=0}^{\ell-1} M^k \Delta M^{\ell-k-1} \:, \:\:  [D (M^\ell)^\T](\Delta) = \sum_{k=0}^{\ell-1} (M^{\ell-k-1})^\T \Delta^\T (M^k)^\T \:.
\end{align*}
By the chain rule,
\begin{align*}
  [D L(K)^\ell](\Delta) = \sum_{k=0}^{\ell-1} L(K)^k \Bstar \Delta L(K)^{\ell-k-1} \:.
\end{align*}
Hence by the chain rule again,
\begin{align*}
  [D \Tr(L(K)^\ell \Sigma (L^\ell)^\T)](\Delta) &= \Tr\left( L^\ell \Sigma \sum_{k=0}^{\ell-1} (L^{\ell-k-1})^\T \Delta^\T \Bstar^\T (L^k)^\T \right) \\
  &\qquad+ \Tr\left( \sum_{k=0}^{\ell-1} L^k \Bstar \Delta L^{\ell-k-1} \Sigma (L^\ell)^\T \right) \\
  &= 2\bigip{\sum_{k=0}^{\ell-1} \Bstar^\T (L^k)^\T L^\ell \Sigma (L^{\ell-k-1})^\T }{\Delta} \:.
\end{align*}
We have shown that:
\begin{align*}
  \nabla_K \Tr( L(K)^\ell \Sigma (L(K)^\ell)^\T ) = 2 \sum_{k=0}^{\ell-1} \Bstar^\T (L^k)^\T L^\ell \Sigma (L^{\ell-k-1})^\T  \:.
\end{align*}
Therefore we can compute the gradient of $J_{\Sigma}(K)$ as:
\begin{align*}
  \nabla J_{\Sigma}(K) = 2 (T-1) \Bstar^\T L \Sigma + 2 \sum_{\ell=2}^{T-1}  \sum_{k=0}^{\ell-1}(T-\ell)  \Bstar^\T (L^k)^\T L^\ell \Sigma (L^{\ell-k-1})^\T \:.
\end{align*}
Now observe that $L(K) = \Bstar(K - K_\star)$ and therefore:
\begin{align*}
  \ip{\nabla J_{\Sigma}(K)}{K - \Kstar} &= \Tr(\nabla J_{\Sigma}(K) (K - \Kstar)^\T) \\
  &= 2 (T-1) \Tr(L \Sigma L^\T) + 2 \sum_{\ell=2}^{T-1}  \sum_{k=0}^{\ell-1}(T-\ell)  \Tr(L^{\ell} \Sigma (L^\ell)^\T) \\
  &\stackrel{(a)}{\geq} 2 (T-1) \Tr(L \Sigma L^\T) \\
  &\geq 2(T-1) \lambda_{\min}(\Sigma) \sigma_{\min}(\Bstar)^2 \norm{K - K_\star}_F^2 \:.
\end{align*}
Above, (a) follows since $\Tr(AB) \geq 0$ for positive semi-definite matrices $A,B$.
This condition proves that $K = \Kstar$ is the unique stationary point,
and establishes the restricted strong convexity $\mathsf{RSC}(m, \R^{d \times n})$ condition for $J_{\Sigma}(K)$
with constant $m = 2(T-1) \lambda_{\min}(\Sigma) \sigma_{\min}(\Bstar)^2$.

Finally, we show that the Hessian of $J_{\Sigma}(K)$ evaluated at $\Kstar$ is positive definite.
Fix a test matrix $H \in \R^{d \times n}$, and define the function
$g(t) := \ip{H}{\nabla J_{\Sigma}(K_\star + t H)}$.
By standard properties of the directional derivative, we have that
$\mathrm{Hess} J_{\Sigma}(\Kstar)[H, H] = g'(0)$.
Observing that $L(\Kstar + t H) = t \cdot \Bstar H$,
we have that:
\begin{align*}
  g(t) &= 2(T-1) t \Tr(\Sigma H^\T \Bstar^\T \Bstar H) \\
  &\qquad+ 2 \sum_{\ell=2}^{T-1} \sum_{k=0}^{\ell-1} (T-\ell) t^{2\ell-1} \Tr(H^\T \Bstar^\T (H^\T \Bstar^\T)^k (\Bstar H)^\ell \Sigma (H^\T \Bstar^\T)^{\ell-k-1}  ) \:,
\end{align*}
from which we conclude:
\begin{align*}
  \mathrm{Hess} J_{\Sigma}(\Kstar)[H, H] = 2 (T-1) \Tr( \Sigma H^\T \Bstar^\T \Bstar H) = 2 (T-1) \vec(H)^\T (\Sigma \otimes \Bstar^\T \Bstar) \vec(H) \:.
\end{align*}

\subsection{Proof of Theorem~\ref{thm:policy_opt_plugin_risk}}
\label{sec:pg:policy_opt_plugin_risk}

Recal that the pair $(A, B)$ is stabilizable if there exists a feedback matrix $K$ such that
$\rho(A + BK) < 1$.
We first state a result which gives a sufficient condition for the existence of a unique positive definite
solution to the discrete algebraic Riccati equation.
\begin{lem}[Theorem 2, \citet{molinari75}]
\label{lem:dare_existence}
Suppose that $Q \succ 0$, $(A, B)$ is stabilizable, and $B$ has full column rank.
Then there exists a unique positive definite solution $P$ to the DARE:
\begin{align}
    P = A^\T P A - A^\T P B (B^\T P B)^{-1} B^\T P A + Q \:. \label{eq:dare_Q}
\end{align}
This $P$ satisfies the lower bound $P \succeq Q$, and if $A$ is contractive (i.e. $\norm{A} < 1$)
satisfies the upper bound $\norm{P} \leq \frac{\norm{Q}}{1-\norm{A}^2}$.
\end{lem}
\begin{proof}
Define the map $\Psi(z;A) := B^\T (z^{-1} I_n - A)^{-\T} Q (z I_n - A)^{-1} B$.
Let $K$ be such that $A+BK$ is stable.
We observe that for $\abs{z} = 1$,
we have that:
\begin{align*}
    \Psi(z; A+BK) = B^* (zI_n - (A+BK))^{-*} Q (zI_n - (A+BK))^{-1} B \succ 0\:.
\end{align*}
This is because $Q \succ 0$, $B^* B \succ 0$, and the matrix $zI_n - (A+BK)$ does not drop rank
since $A+BK$ has no eigenvalues on the unit circle.
Therefore by Theorem 2 of \citet{molinari75},
there exists a unique symmetric solution $P$ that satisfies \eqref{eq:dare_Q} with the additional
constraint that $B^\T P B \succ 0$ and that $\rho(A_c) < 1$ with $A_c := A - B (B^\T P B)^{-1} B^\T P A$.
But \eqref{eq:dare_Q} means that:
\begin{align*}
    A_c^\T P A_c &=  (A - B (B^\T P B)^{-1} B^\T P A)^\T P (A - B (B^\T P B)^{-1} B^\T P A) \\
    &= A^\T P A - A^\T P B (B^\T P B)^{-1} B^\T P A - A^\T P B (B^\T P B)^{-1} B^\T P A \\
    &\qquad+ A^\T P B (B^\T P B)^{-1} B^\T P B (B^\T P B)^{-1} B^\T P A \\
    &= A^\T P A - A^\T P B (B^\T P B)^{-1} B^\T P A \\
    &= P - Q \:.
\end{align*}
Hence, we have $A_c^\T P A_c - P + Q = 0$, and since $A_c$ is stable by Lyapunov theory
we know that $P \succeq Q$.
Furthermore, since $P \succeq 0$, \eqref{eq:dare_Q} implies that
$P \preceq A^\T P A + Q$ from which the upper bound on $\norm{P}$ follows
under the contractivity assumptions.
\end{proof}
Next, we state a result which gives the derivative of the discrete algebraic Riccati equation.
\begin{lem}[Section A.2 of \citet{abeille17}]
\label{lem:dare_derivative}
Let $(Q, R)$ be positive semidefinite matrices.
Suppose that $(A, B)$ are such that
there exists a unique positive definite solution $P(A, B)$
to $\mathsf{dare}(A, B, Q, R)$.
For a perturbation $\rvectwo{\Delta_A}{\Delta_B} \in \R^{n \times (n+d)}$, we have that the
Fr{\'{e}}chet derivative $[D_{(A,B)} P(A, B)]$ evaluated at the perturbation $\rvectwo{\Delta_A}{\Delta_B}$
is given by:
\begin{align*}
  [D_{(A,B)} P(A, B)](\rvectwo{\Delta_A}{\Delta_B}) = \dlyap\left(A_c, A_c^\T P\rvectwo{\Delta_A}{\Delta_B} \cvectwo{I_n}{K} + \cvectwo{I_n}{K}^\T \rvectwo{\Delta_A}{\Delta_B}^\T P A_c\right) \:,
\end{align*}
where $P = P(A, B)$, $K = -(B^\T P B + R)^{-1} B^\T P A$, and $A_c = A + BK$.
\end{lem}

With these two lemmas, we are ready to proceed.
We differentiate the map $h(A, B) := -(B^\T P(A, B) B + R)^{-1} B^\T P(A, B) A$.
By the chain rule:
\begin{align*}
    &[D_{(A,B)} h(A, B)](\Delta) = -(B^\T P B + R)^{-1}(B^\T P \Delta_A + \Delta_B^\T P A + B^\T [D_{(A,B)} P](\Delta) A ) \\
  &\qquad +(B^\T P B + R)^{-1} ( \Delta_B^\T P B + B^\T P \Delta_B + B^\T [D_{(A,B)} P](\Delta) B   )(B^\T P B + R)^{-1} B^\T P A \:.
\end{align*}
We now evaluate this derivative at:
\begin{align*}
    A &=  \Astar \:, B = \Bstar \:, Q = I_n \:, R = 0 \:.
\end{align*}
Note that $P(A, B) = I_n$ and also by Lemma~\ref{lem:dare_derivative}, we have that $[D_{(A,B)} P(A, B)] = 0$,
since $A_c = 0$.
Therefore the derivative $[D_{(A,B)} h(A, B)](\Delta)$ simplifies to:
\begin{align*}
  [D_{(A,B)} h(A, B)](\Delta) &= - (\Bstar^\T \Bstar)^{-1} (\Bstar^\T \Delta_A + \Delta_B^\T \Astar) + (\Bstar^\T \Bstar)^{-1} (\Delta_B^\T \Bstar + \Bstar^\T \Delta_B) \Bstar^{\dag} \Astar \\
  &= - \Bstar^{\dag}\Delta_A + \Bstar^{\dag} \Delta_B \Bstar^{\dag} \Astar \:.
\end{align*}
Hence we have:
\begin{align*}
  \vec([D_{(A,B)} h(A, B)](\Delta) ) &= \rvectwo{ -(I_n \otimes \Bstar^{\dag})  }{ (\Bstar^{\dag} \Astar)^\T \otimes \Bstar^{\dag}  }\vec(\Delta) \:.
\end{align*}
Now using the assumption that $\Astar$ is stable, from Lemma~\ref{lem:ls_driven_asymptotic_dist}
we have that by the delta method:
\begin{align*}
  &\sqrt{N} \vec(h(\Ah(N), \Bh(N)) - \Kstar) \\
  &\qquad\distconv \calN\left( 0,  \frac{\sigma_w^2}{T} \rvectwo{ -(I_n \otimes \Bstar^{\dag})  }{ (\Bstar^{\dag} \Astar)^\T \otimes \Bstar^{\dag}  } \left(\bmattwo{P_\infty^{-1}}{0}{0}{(1/\sigma_u^2) I_d} \otimes I_n\right)  \cvectwo{ -(I_n \otimes (\Bstar^{\dag})^\T)  }{ \Bstar^{\dag} \Astar \otimes (\Bstar^{\dag})^\T  }  + o(1/T) \right) \\
  &\qquad=: \varphi \:.
\end{align*}

We now make use of the second order delta method.
Recall that the Hessian of $J$ at $\Kstar$ is
$\mathrm{Hess} J(\Kstar) [H,H] = 2 (T-1) \sigma_w^2 \ip{H}{\Bstar^\T \Bstar H}$.
If $\sqrt{N} \vec(\Kh(N) - \Kstar) \distconv \varphi$, then by the second order delta method:
\begin{align*}
  N \cdot (J(\Kh(N)) - J_\star) \distconv (T-1) \sigma_w^2 \varphi^\T (I_n \otimes \Bstar^\T \Bstar) \varphi \:.
\end{align*}
Next we make an intermediate calculation:
\begin{align*}
  &\cvectwo{ -(I_n \otimes (\Bstar^{\dag})^\T)  }{ \Bstar^{\dag} \Astar \otimes (\Bstar^{\dag})^\T  } (I_n \otimes \Bstar^\T \Bstar)\rvectwo{ -(I_n \otimes \Bstar^{\dag})  }{ (\Bstar^{\dag} \Astar)^\T \otimes \Bstar^{\dag}  } \\
  &= \bmattwo{ I_n \otimes \Bstar\Bstar^{\dag} }{ -((\Bstar^{\dag} \Astar)^\T \otimes \Bstar\Bstar^{\dag})  }{ - (\Bstar^{\dag} \Astar \otimes \Bstar\Bstar^{\dag})  }{ \Bstar^{\dag} \Astar\Astar^\T (\Bstar^{\dag})^\T \otimes \Bstar\Bstar^{\dag}  } \\
  &= \bmattwo{ I_n }{ - (\Bstar^{\dag} \Astar)^\T }{ - \Bstar^{\dag} \Astar }{ \Bstar^{\dag} \Astar\Astar^\T (\Bstar^{\dag})^\T } \otimes \Bstar\Bstar^{\dag} \:.
\end{align*}

Let $Z_N := N \cdot ( J(\Kh(N)) - J_\star )$.
To conclude the proof, we show that the sequence $\{ Z_N \}$ is uniformly integrable.
Once we have the uniform integrability in place, then by Lemma~\ref{lem:lp_convergence}:
\begin{align*}
  &\lim_{N \to \infty} N \cdot ( J(\Kh(N)) - J_\star ) \\
  &\qquad= \sigma_w^4\frac{T-1}{T}  \Tr\left( \left(\bmattwo{P_\infty^{-1}}{0}{0}{(1/\sigma_u^2) I_d} \otimes I_n\right) \left( \bmattwo{ I_n }{ - (\Bstar^{\dag} \Astar)^\T }{ - \Bstar^{\dag} \Astar }{ \Bstar^{\dag} \Astar\Astar^\T (\Bstar^{\dag})^\T } \otimes \Bstar\Bstar^{\dag}   \right)    \right) + o_T(1) \\
  &\qquad= \sigma_w^4 \frac{T-1}{T} \Tr\left( \bmattwo{P_\infty^{-1}}{0}{0}{(1/\sigma_u^2) I_d} \bmattwo{ I_n }{ - (\Bstar^{\dag} \Astar)^\T }{ - \Bstar^{\dag} \Astar }{ \Bstar^{\dag} \Astar\Astar^\T (\Bstar^{\dag})^\T }\right) \Tr(\Bstar\Bstar^{\dag}) + o_T(1) \\
  &\qquad= \sigma_w^4 \frac{T-1}{T} \left(\Tr(P_\infty^{-1}) + \frac{\norm{ \Bstar^{\dag} \Astar }_F^2}{\sigma_u^2}  \right)d + o_T(1) \:.
\end{align*}
%

To conclude the proof, let $C_\star, \rho_\star$ be such that
$\norm{\Astar^k} \leq C_\star \rho_\star^k$ with $\rho_\star \in (0, 1)$: these constants exist because $\Astar$ is stable.
Now define the events:
\begin{align*}
    \calE_{\mathsf{Alg}} &:= \{ \rho(\Ah(N)) \leq \varrho \:, \:\: \norm{\Ah(N)} \leq \zeta \:, \:\: \norm{\Bh(N)} \leq \psi \:, \:\: \sigma_{d}(\Bh(N)) \geq \gamma \} \:, \\
    \calE_{\mathsf{Bdd}} &:= \{ \norm{\Ah(N) - \Astar} \leq \frac{1-\rho_\star}{2C_\star} \:, \:\: \norm{\Bh(N) - \Bstar} \leq \sigma_{d}(\Bstar)/2 \} \:.
\end{align*}
Fix a finite $p \geq 1$.  We write:
\begin{align*}
  \E[Z_N^p] &= N^p \E[ (J(\Kh(N)) - J_\star)^p \ind_{\calE_{\mathsf{Bdd}}} ] + N^p \E[ (J(\Kh(N)) - J_\star)^p \ind_{\calE_{\mathsf{Bdd}}^c} ] \\
  &=N^p \E[ (J(\Kh(N)) - J_\star)^p \ind_{\calE_{\mathsf{Bdd}}} ] + N^p \E[ (J(\Kh(N)) - J_\star)^p \ind_{\calE_{\mathsf{Bdd}}^c \cap \calE_{\mathsf{Alg}}  } ] \\
    &\qquad+ N^p \E[ (J(\Kh(N)) - J_\star)^p \ind_{\calE_{\mathsf{Bdd}}^c \cap \calE_{\mathsf{Alg}}^c  } ] \\
  &=N^p \E[ (J(\Kh(N)) - J_\star)^p \ind_{\calE_{\mathsf{Bdd}}} ] + N^p \E[ (J(\Kh(N)) - J_\star)^p \ind_{\calE_{\mathsf{Bdd}}^c \cap \calE_{\mathsf{Alg}}  } ] \\
    &\qquad+ N^p (J(0) - J_\star)^p \Pr( \calE_{\mathsf{Bdd}}^c \cap \calE_{\mathsf{Alg}}^c ) \\
  &\leq N^p \E[ (J(\Kh(N)) - J_\star)^p \ind_{\calE_{\mathsf{Bdd}}} ] + N^p \E[ (J(\Kh(N)) - J_\star)^p \ind_{\calE_{\mathsf{Bdd}}^c \cap \calE_{\mathsf{Alg}}  } ] + N^p (J(0) - J_\star)^p \Pr( \calE_{\mathsf{Bdd}}^c ) \:.
\end{align*}
We now consider what happens on these three events.
For the remainder of the proof, we let $C$ denote a constant that depends on
$n, d, p,C_\star, \rho_\star, \varrho, \zeta, \psi, \gamma, \Astar, \Bstar, T, \varepsilon, \sigma_w^2, \sigma_u^2$ but not on $N$,
whose value can change from line to line.

\paragraph{On the event $\calE_{\mathsf{Bdd}}$.}
By a Taylor expansion we write:
\begin{align*}
    h(\Ah(N), \Bh(N)) - h(\Astar, \Bstar) &= [D_{(A,B)} h(\tilde{A}, \tilde{B})]\left( \rvectwo{\Ah(N) - \Astar}{\Bh(N) - \Bstar} \right) \:,
\end{align*}
where $\tilde{A} = t \Astar + (1-t) \Ah(N)$
and $\tilde{B} = t \Bstar + (1-t) \Bh(N)$ for some $t \in [0, 1]$.
Observe that on $\calE_{\mathsf{Bdd}}$, we have that
\begin{align*}
    \tilde{A}, \tilde{B} \in \mathcal{G} := \left\{ (A, B) : \norm{A} \leq \norm{\Astar} + \frac{1-\rho_\star}{2C_\star}  \:, \norm{B} \leq \norm{\Bstar} + \sigma_{d}(\Bstar)/2 \:,\sigma_d(B) \geq \sigma_{d}(\Bstar)/2 \right\} \:.
\end{align*}
By Lemma~\ref{lem:simple_perturbation}, each $(A, B) \in \mathcal{G}$ is stabilizable (since $A$ is stable)
and $B$ has full column rank.
Therefore by Lemma~\ref{lem:dare_existence}, for any $(A, B) \in \mathcal{G}$ we have that $\mathsf{dare}(A, B, I_n, 0)$
has a unique positive definite solution and its derivative is well defined.
By the compactness of $\mathcal{G}$ and the continuity of $h$ and its derivative, we define the finite constants
\begin{align*}
    C_K := \sup_{A,B \in \mathcal{G}} \norm{h(A, B)} \:, \:\: C_{\mathrm{deriv}} := \sup_{A, B \in \mathcal{G}} \norm{ [D_{(A,B)} h(A, B)] }  \:.
\end{align*}
We can now Taylor expand $J(K)$ around $\Kstar$ and obtain:
\begin{align*}
  J(\Kh(N)) - J_\star &= \frac{1}{2} \mathrm{Hess} J(\tilde{K})[\Kh(N) - \Kstar, \Kh(N) - \Kstar] \\
    &\leq \frac{1}{2} \left(\sup_{\norm{\Ktilde} \leq C_K + \norm{\Kstar}} \norm{\mathrm{Hess} J(\tilde{K})}\right) \norm{\Kh(N) - \Kstar}_F^2 \\
    &\leq \frac{d}{2} \left(\sup_{\norm{\Ktilde} \leq C_K + \norm{\Kstar}} \norm{\mathrm{Hess} J(\tilde{K})}\right) C_{\mathrm{deriv}}^2 (\norm{\Ah(N) - \Astar}^2 + \norm{\Bh(N) - \Bstar}^2) \:.
\end{align*}
Hence for $N$ sufficiently large, by Lemma~\ref{lem:moment_control_driven_ls} we have
\begin{align*}
  N^p \cdot \E[ (J(\Kh(N)) - J_\star)^p \ind_{\calE_{\mathsf{Bdd}}} ] &\leq C N^p ( \E[ \norm{\Ah(N) - \Astar}^{2p}] +  \E[ \norm{\Bh(N) - \Bstar}^{2p}]  ) \\
  &\leq C N^{p} ( \frac{1}{N^p} ) = C \:.
\end{align*}

\paragraph{On the event $\calE_{\mathsf{Bdd}}^c \cap \calE_{\mathsf{Alg}}$.}

In this case, we use the bounds given by $\calE_{\mathsf{Alg}}$ to bound the controller $\Kh(N)$.
Lemma~\ref{lem:dare_existence} ensures that the solution $\widehat{P} = \mathsf{dare}(\Ah(N), \Bh(N), I_n, 0)$
exists and satisfies $\widehat{P} \succeq I_n$.
Let the finite constant $C_P$ be $C_P := \sup_{\rho(A) \leq \varrho, \norm{A} \leq \zeta, \norm{B} \leq \psi, \sigma_{d}(B) \geq \gamma} \norm{\mathsf{dare}(A, B, I_n, 0)}$.
We can then bound $\norm{\Kh(N)}$ as follows. Dropping the indexing of $N$,
\begin{align*}
    \norm{\Kh} &= \norm{ (\Bh^\T \widehat{P} \Bh)^{-1} \Bh^\T \widehat{P} \Ah } \leq \frac{1}{\sigma_{\min}( \Bh^\T \widehat{P} \Bh )} \norm{\Bh^\T \widehat{P} \Ah} \leq \frac{C_P \psi \zeta}{\gamma^2}\:.
\end{align*}
Therefore:
\begin{align*}
    N^p \cdot \E[ (J(\Kh(N)) - J_\star)^p \ind_{\calE^c_{\mathsf{Bdd}} \cap \calE_{\mathsf{Alg}}} ]  \leq N^p \cdot \left( \sup_{\norm{K} \leq \frac{C_P \psi\zeta}{\gamma^2}   } (J(K) - J_\star)^p \right) \Pr(\calE^c_{\mathsf{Bdd}}) \leq C N^p \Pr(\calE^c_{\mathsf{Bdd}})  \:.
\end{align*}
By Lemma~\ref{lem:moment_control_driven_ls}, we can choose $N$
large enough
such that $\Pr(\calE_{\mathsf{Bdd}}^c) \leq 1/N^p$
so that
$N^p \cdot \E[ (J(\Kh(N)) - J_\star)^p \ind_{\calE^c_{\mathsf{Bdd}} \cap \calE_{\mathsf{Alg}}} ]  \leq C$.

\paragraph{On the event $\calE_{\mathsf{Bdd}}^c \cap \calE_{\mathsf{Alg}}^c$.}
This case is simple. We simply invoke Lemma~\ref{lem:moment_control_driven_ls}
to choose an $N$ large enough
such that $\Pr(\calE_{\mathsf{Bdd}}^c) \leq 1/(N(J(0) - J_\star))^p$.

\paragraph{Putting it together.}
If we take $N$ as the maximum over the three cases described above, we have hence shown that
for all $N$ greater than this constant:
\begin{align*}
  \E[ Z_N^p ] \leq C \:.
\end{align*}
This shows the desired uniform integrability condition for $Z_N$.
The asymptotic bound now follows from Lemma~\ref{lem:lp_convergence}.

\subsection{Proof of Theorem~\ref{thm:policy_opt_pg_risk}}
\label{sec:pg:policy_opt_pg}

The proof works by applying Lemma~\ref{lem:sgd_asymptotics}
with the function $F(\theta) = J_{\Sigma}(K)$ with $\Sigma = \sigma_u^2 \Bstar\Bstar^\T + \sigma_w^2 I_n$
and $G(\theta) = J(K)$.
We first need to verify the hypothesis of the lemma.
We define the convex domain $\Theta$ as $\Theta = \{ K \in \R^{d \times n} : \norm{K} \leq \zeta \}$.
Note that $\Kstar$ is in the interior of $\Theta$,
since we assume that $\norm{\Kstar} < \zeta$.
Recall that the policy gradient $g(K; \xi)$ is:
\begin{align*}
  g(K; \xi) = \frac{1}{\sigma_u^2} \sum_{t=1}^{T-1} \eta_t x_t^\T \Psi_t \:, \:\: \xi = (\eta_0, w_0, \eta_1, w_1, ..., \eta_{T-1}, w_{T-1}) \:.
\end{align*}
It is clear that $x_t$ is a polynomial in $(K, \xi)$. Furthermore, all three of the $\Psi_t$'s we study are
also polynomials in $(K, \xi)$. Hence $[D_K g(K; \xi)]$ is a matrix with entries that are polynomial in $(K, \xi)$.
Therefore, for every $\xi$, for all fixed $K_1, K_2 \in \Theta$,
\begin{align*}
    \norm{ g(K_1; \xi) - g(K_2; \xi) }_F \leq \sup_{K \in \Theta} \norm{[D_K g(K; \xi)]}_F \norm{K_1 - K_2}_F \:.
\end{align*}
Hence squaring and taking expectations,
\begin{align*}
    \E_\xi[ \norm{ g(K_1; \xi) - g(K_2; \xi) }_F^2 ] \leq \E_\xi\left[ \sup_{K \in \Theta} \norm{[D_K g(K; \xi)]}_F^2 \right] \norm{K_1 - K_2}_F^2 \:.
\end{align*}
We can now define the constant $L := \E_\xi\left[ \sup_{K \in \Theta} \norm{[D_K g(K; \xi)]}_F^2 \right]$.
To see that this quantity $L$ is finite, observe that $\norm{[D_K g(K; \xi)]}_F^2$ is a polynomial of $\xi$ with coefficients
given by $K$ (and $\Astar, \Bstar$). Since $K$ lives in a compact set $\Theta$, these coefficients are uniformly bounded
and hence the their moments are bounded.
In Section~\ref{sec:pg:calcs}, we showed that
the function $J_{\Sigma}(K)$ satisfies the $\mathsf{RSC}(m, \Theta)$ condition
with $m = 2 (T-1)\sigma_w^2 \sigma_{\min}(\Bstar)^2$.
Also it is clear that the high probability bound
on $\norm{g(K; \xi)}_F$ can be achieved by standard Gaussian concentration results.
Hence by Lemma~\ref{lem:sgd_asymptotics}, and in particular Equation~\ref{eq:cost_lower_bound_sgd},
\begin{align}
  \liminf_{N \to \infty} N \cdot \E[ J(\Kh) - J_\star ] &\geq \frac{ \E_{\xi}[ \norm{g(\Kstar; \xi)}_F^2 ]}{8 (T-1) \sigma_w^2 \sigma_{\min}(\Bstar)^2 \lambda_{\max}( (\nabla^2 J(\Kstar))^{-1} ( \nabla^2 J_{\Sigma}(\Kstar) - \frac{m}{2} I_{nd} ))   } \nonumber \\
    &= \frac{\E_{\xi}[ \norm{g(\Kstar; \xi)}_F^2 ]}{ 8 (T-1) \sigma_{\min}(\Bstar)^2 (\sigma_w^2 + \sigma_u^2 \norm{\Bstar}^2) } \label{eq:pg_expr} \:.
\end{align}
Above, the inequality holds since
we have that,
\begin{align*}
  \nabla^2 J(\Kstar) &= 2(T-1) (\sigma_w^2I_n \otimes \Bstar^\T \Bstar) \:, \\
  \nabla^2 J_{\Sigma}(\Kstar) &= 2(T-1)( (\sigma_w^2 I_n + \sigma_u^2 \Bstar\Bstar^\T) \otimes \Bstar^\T \Bstar) = \nabla^2 J(K_\star) + 2(T-1)\sigma_u^2 (\Bstar\Bstar^\T \otimes \Bstar^\T \Bstar) \:,
\end{align*}
and therefore,
\begin{align*}
  (\nabla^2 J(\Kstar))^{-1} ( \nabla^2 J_{\Sigma}(\Kstar) - \frac{m}{2} I_{nd} ) &= I_{nd} + \frac{\sigma_u^2}{\sigma_w^2} (\Bstar\Bstar^\T \otimes I_d) - \frac{\sigma_{\min}(\Bstar)^2}{2} (I_n \otimes (\Bstar^\T \Bstar)^{-1})  \\
  &\preceq I_{nd} + \frac{\sigma_u^2}{\sigma_w^2} (\Bstar\Bstar^\T \otimes I_d) \:.
\end{align*}

The remainder of the proof is to estimate
the quantity $\E_\xi[ \norm{g(\Kstar; \xi)}^2_F ]$.
Note that at $K = \Kstar$, $x_t = \Bstar \eta_{t-1} + w_{t-1}$ since the dynamics are cancelled out.
Define $c_{t \to T} := \sum_{\ell=t}^{T} \norm{x_\ell}_2^2$.
At $K = \Kstar$, we have $c_{t \to T} = \sum_{\ell=t-1}^{T-1} \norm{\Bstar \eta_\ell + w_\ell}_2^2$.
Observe that we have for $t_2 > t_1$, for any $h$ that depends on only
$(\eta_{t_1}, w_{t_1}, \eta_{t_1+1}, w_{t_1+1}, ...)$:
\begin{align*}
  \E[\ip{\eta_{t_1}}{\eta_{t_2}} \ip{x_{t_1}}{x_{t_2}} h] &= \E[ \ip{\eta_{t_1}}{\eta_{t_2}} ( \ip{\Bstar \eta_{t_1-1}}{\Bstar \eta_{t_2-1}} + \ip{w_{t_1-1}}{w_{t_2-1}} \\
  &\qquad+ \ip{\Bstar \eta_{t_1-1}}{w_{t_2-1}} + \ip{\Bstar \eta_{t_2-1}}{w_{t_1-1}} ) h] \\
    &= 0 \:.
\end{align*}
As a consequence, we have that as long as $\Psi_t$ only depends on $(\eta_t, w_t, \eta_{t+1}, w_{t+1},...)$:
\begin{align*}
  \E[\norm{g(K; \xi)}_F^2 ] &= \frac{1}{\sigma_u^4} \sum_{t=1}^{T-1} \E[\norm{\eta_t}^2_2 \norm{x_t}^2_2 \Psi_t^2 ] + \frac{2}{\sigma_u^4} \sum_{t_2 > t_1 = 1}^{T-1} \E[ \ip{\eta_{t_1}}{\eta_{t_2}} \ip{x_{t_1}}{x_{t_2}} \Psi_{t_1} \Psi_{t_2} ] \\
  &= \frac{1}{\sigma_u^4} \sum_{t=1}^{T-1} \E[\norm{\eta_t}^2_2 \norm{x_t}^2_2 \Psi_t^2 ] \:.
\end{align*}

\subsubsection{Simple baseline}

Recall that the simple baseline is to set $b_t(x_t; K) = \norm{x_t}_2^2$. Hence, the policy gradient
estimate simplifies to
$g(K;\xi) = \frac{1}{\sigma_u^2} \sum_{t=1}^{T-1} \eta_t x_t^\T c_{t+1 \to T}$.
Since we have that $c_{t+1 \to T}$ at optimality only depends only on $(\eta_t, w_t, \eta_{t+1}, w_{t+1}, ...)$,
we compute $\E[\norm{g(\Kstar; \xi)}_F^2]$ as follows:
\begin{align*}
  &\E[ \norm{g(\Kstar; \xi)}^2_F ] = \frac{1}{\sigma_u^4} \sum_{t=1}^{T-1} \E[ \norm{\eta_t}_2^2 \norm{x_t}_2^2 c_{t+1 \to T}^2 ] \\
    &\qquad= \frac{1}{\sigma_u^4} \sum_{t=1}^{T-1} \E\left[ \norm{\eta_t}_2^2 \norm{\Bstar \eta_{t-1} + w_{t-1}}_2^2 \left( \sum_{\ell=t}^{T-1} \norm{\Bstar \eta_\ell + w_\ell}_2^4 + 2 \sum_{\ell_2 > \ell_1=t}^{T-1} \norm{\Bstar \eta_{\ell_1} + w_{\ell_1}}_2^2 \norm{\Bstar \eta_{\ell_2} + w_{\ell_2}}_2^2 \right) \right] \\
    &\qquad= \frac{1}{\sigma_u^4} \sum_{t=1}^{T-1} \E[ \norm{\Bstar \eta_{t-1} + w_{t-1}}_2^2 \norm{\eta_t}_2^2 \norm{\Bstar \eta_t + w_t}_2^4 ] \\
    &\qquad\qquad + \frac{1}{\sigma_u^4} \sum_{t=1}^{T-1} \sum_{\ell=t+1}^{T-1} \E[ \norm{\Bstar \eta_{t-1} + w_{t-1}}_2^2 \norm{\eta_t}_2^2 \norm{\Bstar \eta_\ell + w_\ell}_2^4 ] \\
    &\qquad\qquad + \frac{2}{\sigma_u^4} \sum_{t=1}^{T-1} \sum_{\ell_2>t}^{T-1} \E[ \norm{\Bstar \eta_{t-1} + w_{t-1}}_2^2 \norm{\eta_t}_2^2 \norm{\Bstar \eta_t + w_t}_2^2 \norm{\Bstar \eta_{\ell_2} + w_{\ell_2}}_2^2 ] \\
    &\qquad\qquad + \frac{2}{\sigma_u^4} \sum_{t=1}^{T-1} \sum_{\ell_2 > \ell_1=t+1}^{T-1} \E[ \norm{\Bstar \eta_{t-1} + w_{t-1}}_2^2 \norm{\eta_t}_2^2 \norm{\Bstar \eta_{\ell_1} + w_{\ell_1}}_2^2 \norm{\Bstar \eta_{\ell_2} + w_{\ell_2}}_2^2 ] \\
    &\qquad= \frac{2}{\sigma_u^4} \sum_{t=1}^{T-1} \sum_{\ell_2 > \ell_1=t+1}^{T-1} \E[ \norm{\Bstar \eta_{t-1} + w_{t-1}}_2^2 \norm{\eta_t}_2^2 \norm{\Bstar \eta_{\ell_1} + w_{\ell_1}}_2^2 \norm{\Bstar \eta_{\ell_2} + w_{\ell_2}}_2^2 ] + o(T^3) \\
    &\qquad= \frac{2}{\sigma_u^4} \sum_{t=1}^{T-1} \sum_{\ell_2 > \ell_1=t+1}^{T-1} \sigma_u^2 d ( \E[ \norm{\Bstar \eta_0 + w_0}_2^2 ])^3 + o(T^3) \\
    &\qquad\asymp T^3 \frac{1}{\sigma_u^2} d( \sigma_u^2 \norm{\Bstar}_F^2 + \sigma_w^2 n )^3 + o(T^3) \:.
\end{align*}

\subsubsection{Value function baseline}

Recall that the value function at time $t$ for a particular policy $K$ is defined as:
\begin{align*}
  V^K_t(x) = \E\left[ \sum_{\ell=t}^{T} \norm{x_\ell}_2^2 \bigg| x_t = x \right] \:.
\end{align*}
We now consider policy gradient with the value function baseline $b_t(x_t; K) = V^K_t(x_t)$:
\begin{align*}
  g(K;\xi) = \frac{1}{\sigma_u^2} \sum_{t=1}^{T-1} \eta_t x_t^\T (c_{t \to T} - V_t^K(x_t)) \:.
\end{align*}
Recalling that under $\Kstar$ the dynamics are cancelled out, we readily compute:
\begin{align*}
  V_t^{\Kstar}(x) = \norm{x}_2^2 + (T-t)(\sigma_u^2 \norm{\Bstar}_F^2 + \sigma_w^2 n) \:.
\end{align*}
Therefore:
\begin{align*}
  g(\Kstar; \xi) = \frac{1}{\sigma_u^2} \sum_{t=1}^{T-1} \eta_t x_t^\T (c_{{t+1} \to T} - (T-t)(\sigma_u^2\norm{\Bstar}_F^2 + \sigma_w^2 n)) \:.
\end{align*}
Define $\beta := \sigma_u^2 \norm{\Bstar}_F^2 + \sigma_w^2 n$.
We compute the variance as:
\begin{align*}
  &\E[\norm{g(\Kstar; \xi)}_F^2] = \frac{1}{\sigma_u^4} \sum_{t=1}^{T-1} \E\Bigg[ \norm{\eta_t}_2^2 \norm{\Bstar \eta_{t-1} + w_{t-1}}_2^2 \\
  &\qquad\qquad \times\left( \sum_{\ell=t}^{T-1} (\norm{\Bstar \eta_\ell + w_\ell}_2^2 - \beta)^2 + 2 \sum_{\ell_2 > \ell_1=t}^{T-1} (\norm{\Bstar \eta_{\ell_1} + w_{\ell_1}}_2^2 - \beta)( \norm{\Bstar \eta_{\ell_2} + w_{\ell_2}}_2^2 - \beta) \right) \Bigg] \\
  &\qquad= \frac{1}{\sigma_u^4} \sum_{t=1}^{T-1} \E[ \norm{\eta_t}_2^2 \norm{\Bstar \eta_{t-1} + w_{t-1}}_2^2 (\norm{\Bstar \eta_t + w_t}_2^2 - \beta)^2 ] \\
  &\qquad\qquad + \frac{1}{\sigma_u^4} \sum_{t=1}^{T-1} \sum_{\ell=t+1}^{T-1}\E[ \norm{\eta_t}_2^2 \norm{\Bstar \eta_{t-1} + w_{t-1}}_2^2 (\norm{\Bstar \eta_\ell + w_\ell}_2^2 - \beta)^2 ] \\
  &\qquad= \frac{1}{\sigma_u^4} \sum_{t=1}^{T-1} \sum_{\ell=t+1}^{T-1}\E[ \norm{\eta_t}_2^2 \norm{\Bstar \eta_{t-1} + w_{t-1}}_2^2 (\norm{\Bstar \eta_\ell + w_\ell}_2^2 - \beta)^2 ] + o(T^2) \\
  &\qquad\asymp T^2 \frac{d}{\sigma_u^2} (\sigma_u^2 \norm{\Bstar}_F^2 + \sigma_w^2 n)( \E[ \norm{\Bstar \eta_\ell + w_\ell}_2^4 ] - \beta^2 ) + o(T^2) \\
  &\qquad\stackrel{(a)}{\asymp} T^2 \frac{d}{\sigma_u^2} (\sigma_u^2 \norm{\Bstar}_F^2 + \sigma_w^2 n)( \sigma_u^4 \norm{\Bstar^\T \Bstar}_F^2 + \sigma_w^4 n + \sigma_w^2 \sigma_u^2 \norm{\Bstar}_F^2 )  + o(T^2) \:,
\end{align*}
Above, (a) follows because:
\begin{align*}
  \E[ \norm{\Bstar \eta_\ell + w_\ell}_2^4 ] &= 2 (\sigma_u^4 \norm{\Bstar^\T \Bstar}_F^2 + \sigma_w^4 n + 2 \sigma_w^2\sigma_u^2 \norm{\Bstar}_F^2 ) + (\sigma_u^2 \norm{\Bstar}_F^2 + \sigma_w^2 n)^2 \:.
\end{align*}

\subsubsection{Ideal advantage baseline}

Let us first compute $Q_t^{\Kstar}(x_t, u_t)$.
Under $\Kstar$, $x_{\ell+1} = \Bstar \eta_\ell + w_\ell$. So we have:
\begin{align*}
  Q_t^{\Kstar}(x_t, u_t) &= \norm{x_t}_2^2 + \E_{w_t}[\norm{\Astar x_t + \Bstar u_t + w_t}_2^2] + (T - t - 1)( \sigma_u^2 \norm{\Bstar}_F^2 + \sigma_w^2 n) \\
    &= \norm{x_t}_2^2 + \norm{ \Astar x_t + \Bstar u_t}_2^2 + \sigma^2_w n + (T - t - 1)( \sigma_u^2 \norm{\Bstar}_F^2 + \sigma_w^2 n) \:.
\end{align*}
Recalling that $V_t^{\Kstar}(x) = \norm{x}_2^2 + (T-t)(\sigma_u^2 \norm{\Bstar}_F^2 + \sigma_w^2 n)$,
\begin{align*}
  A_t^{\Kstar}(x_t, u_t) &= Q_t^{\Kstar}(x_t, u_t) - V_t^{\Kstar}(x_t) = \norm{ \Astar x_t + \Bstar u_t}_2^2 - \sigma_u^2 \norm{\Bstar}_F^2  \:.
\end{align*}
Therefore, if $u_t = \Kstar x_t + \eta_t$,
we have $A_t^{\Kstar}(x_t, u_t) = \norm{\Bstar \eta_t}_2^2 -  \sigma_u^2 \norm{\Bstar}_F^2$.
Since $A_t^{\Kstar}(x_t, u_t)$ depends only on $\eta_t$,
\begin{align*}
  \E[\norm{g(\Kstar; \xi)}_F^2 ] &= \frac{1}{\sigma_u^4} \sum_{t=1}^{T-1} \E[ \norm{\eta_t}_2^2 \norm{x_t}_2^2 (\norm{\Bstar \eta_t}_2^2 - \sigma_u^2 \norm{\Bstar}_F^2)^2 ] \\
    &= \frac{1}{\sigma_u^4} (T-1) (\sigma_u^2 \norm{\Bstar}_F^2 + \sigma_w^2 n) \E[ \norm{\eta_1}_2^2 (\norm{\Bstar\eta_1}_2^2 - \sigma_u^2 \norm{\Bstar}_F^2)^2 ] \:.
\end{align*}
We have that
$\E[\norm{\eta_1}_2^2] = \sigma_u^2 d$,
$\E[ \norm{\Bstar \eta_1}_2^2\norm{\eta_1}_2^2 ] = \sigma_u^4 (d+2)\norm{\Bstar}_F^2$,
and $\E[\norm{\Bstar\eta_1}_2^4 \norm{\eta_1}_2^2  ] = \sigma_u^6 ( (d+4)\norm{\Bstar}_F^4 + (2d+8) \norm{\Bstar^\T \Bstar}_F^2)$
(this can be computed using Lemma~\ref{lem:sixth_moment}).
Hence,
\begin{align*}
  &\E[ \norm{\eta_1}_2^2 (\norm{\Bstar \eta_1}_2^2 - \sigma_u^2 \norm{\Bstar}_F^2 )^2 ] \\
  &\qquad=\E[ \norm{\Bstar \eta_1}_2^4 \norm{\eta_1}_2^2 + \sigma_u^4 \norm{\Bstar}_F^4 \norm{\eta_1}_2^2 - 2 \sigma_u^2 \norm{\Bstar}_F^2 \norm{\Bstar \eta_1}_2^2 \norm{\eta_1}_2^2 ] \\
  &\qquad=\sigma_u^6 ( (d+4)\norm{\Bstar}_F^4 + (2d+8) \norm{\Bstar^\T \Bstar}_F^2) + \sigma_u^6 \norm{\Bstar}_F^4 d - 2 \sigma_u^6\norm{\Bstar}_F^4 (d+2) \\
  &\qquad= \sigma_u^6 (2d+8) \norm{\Bstar^\T \Bstar}_F^2 \:.
\end{align*}
Therefore,
\begin{align*}
  \E[\norm{g(\Kstar; \xi)}_F^2] &\asymp T (\sigma_u^2 \norm{\Bstar}_F^2 + \sigma_w^2 n) \sigma_u^2 d \norm{\Bstar^\T \Bstar}_F^2 \:.
\end{align*}

\subsubsection{Putting it together}

Combining Equation~\eqref{eq:pg_expr}
with the calculations for $\E_\xi[ \norm{g(\Kstar; \xi)}^2_F ]$, we obtain:
\begin{align*}
  &\liminf_{N \to \infty} N \cdot \E[J(\Khpg(N)) - J_\star] \gtrsim \frac{1}{\sigma_{d}(\Bstar)^2 (\sigma_w^2 + \sigma_u^2 \norm{\Bstar}^2)} \times \\
    &\qquad \begin{cases}
        T^2 \frac{d}{\sigma_u^2} (\sigma_u^2 \norm{\Bstar}_F^2 + \sigma_w^2 n)^3 + o(T^2) &\text{ (Simple baseline)} \\
        T \frac{d}{\sigma_u^2} (\sigma_u^2 \norm{\Bstar}_F^2 + \sigma_w^2 n)( \sigma_u^4 \norm{\Bstar^\T \Bstar}_F^2 + \sigma_w^4 n + \sigma_w^2 \sigma_u^2 \norm{\Bstar}_F^2 ) + o(T) &\text{ (Value function baseline)} \\
        (\sigma_u^2 \norm{\Bstar}_F^2 + \sigma_w^2 n) \sigma_u^2 d \norm{\Bstar^\T \Bstar}_F^2 &\text{ (Advantage baseline)} \\
    \end{cases} \:,
\end{align*}
%
from which Theorem~\ref{thm:policy_opt_pg_risk} follows.

\subsection{Proof of Theorem~\ref{thm:policy_opt_lower_bound}}

Our proof is inspired by lower bounds for the query complexity of
derivative-free optimization of stochastic optimization
(see e.g.~\citet{jamieson12}).

Recall from \eqref{eq:pg:quadratic_growth} that the function $J(K)$ satisfies the quadratic growth condition
$J(K) - J_\star \geq (T-1) \rho^2 \sigma_w^2 \norm{K - \Kstar}_F^2$.
Therefore for any $\vartheta > 0$,
\begin{align*}
    &\inf_{\Kh} \sup_{(\Astar, \Bstar) \in \mathscr{G}(\rho, d)} \E[ J(\Kh) - J_\star ] \\
    &\qquad\geq \inf_{\Kh} \sup_{(\Astar, \Bstar) \in \mathscr{G}(\rho, d)} \: (T-1) \rho^2 \sigma_w^2 \vartheta^2 \cdot \Pr( J(\Kh) - J_\star \geq (T-1) \rho^2 \sigma_w^2 \vartheta^2 ) \\
  &\qquad\geq \inf_{\Kh} \sup_{(\Astar, \Bstar) \in \mathscr{G}(\rho, d)} \: (T-1) \rho^2 \sigma_w^2 \vartheta^2 \cdot \Pr( (T-1) \rho^2 \sigma_w^2 \norm{(-\Ustar^\T) - \Kh}_F^2 \geq (T-1) \rho^2 \sigma_w^2 \vartheta^2 ) \\
  &\qquad= \inf_{\Kh} \sup_{(\Astar, \Bstar) \in \mathscr{G}(\rho, d)} \: (T-1) \rho^2 \sigma_w^2 \vartheta^2 \cdot \Pr( \norm{(-\Ustar^\T) - \Kh}_F \geq \vartheta ) \:.
\end{align*}
Above, the first inequality is Markov's inequality and the second is the quadratic growth condition.

We first state a result regarding the packing number of $O(n, d)$, which we define as:
\begin{align*}
    O(n, d) := \{ U \in \R^{n \times d} : U^\T U = I_d \} \:.
\end{align*}
\begin{lem}
\label{lem:packing}
Let $\delta > 0$, and suppose that $d \leq n/2$.
We have that the packing number $M$ of $O(n, d)$ in the Frobenius norm $\norm{\cdot}_F$ satisfies
\begin{align*}
  M(O(n, d), \norm{\cdot}_F, \delta d^{1/2}) \geq \left(\frac{c}{\delta}\right)^{d(n-d)} \:,
\end{align*}
where $c > 0$ is a universal constant.
\end{lem}
\begin{proof}
Let $G_{n, d}$ denote the Grassman manifold of $d$-dimensional subspaces of $\R^n$.
For two subspaces $E, F \in G_{n, d}$, equip $G_{n, d}$ with the metric
$\rho(E, F) = \norm{P_E - P_F}_F$, where $P_E, P_F$ are the projection matrices onto $E,F$ respectively.
Proposition 8 of \citet{pajor98} tells us
that the covering number $N(G_{n,d}, \rho, \delta d^{1/2}) \geq \left( \frac{c}{\delta} \right)^{d(n-d)}$.
But since $M(G_{n,d}, \rho, \delta d^{1/2}) \geq N(G_{n,d}, \rho, \delta d^{1/2})$, this gives us a lower bound
on the packing number of $G_{n, d}$.
Now for every $E \in G_{n, d}$ we can associate a matrix $E_1 \in O(n, d)$
such that $\Span(E_1) = E$. The projector $P_E$ is simply $P_E = E_1E_1^\T$.
Now let $E, F \in G_{n, d}$ and observe the inequality,
\begin{align*}
  \norm{P_E - P_F}_F = \norm{ E_1E_1^\T - F_1F_1^\T }_F \leq 2 \norm{E_1 - F_1}_F \:.
\end{align*}
Hence a packing of $G_{n, d}$ also yields a packing of $O(n, d)$ up to constant factors.
\end{proof}

Now letting $U_1, ..., U_M$ be a $2\vartheta$-separated set we have
by the standard reduction to multiple hypothesis testing that
that the risk is lower bounded by:
\begin{align}
    (T-1) \rho^2 \sigma_w^2 \vartheta^2 \cdot \inf_{\widehat{V}} \Pr( \widehat{V} \neq V ) &\geq (T-1) \rho^2 \sigma_w^2 \vartheta^2 \cdot \left( 1 - \frac{I(V; Z) + \log{2}}{\log{M}} \right) \:. \label{eq:fano}
\end{align}
where $V$ is a uniform index over $\{1, ..., M\}$ and the inequality is Fano's inequality.

Now we can proceed as follows.
First, we let $U_1, ..., U_M$ be elements of $O(n, d)$ that form a
$2\vartheta \asymp \sqrt{d}$ packing in the $\norm{\cdot}_F$ norm. We know we can
let $M \geq e^{d(n-d)}$ by Lemma~\ref{lem:packing}.
Each $U_i$ induces a covariance $\Sigma_i = \sigma_w^2 I_n + \rho^2 \sigma_u^2 U_iU_i^\T \preceq (\sigma_w^2 + \rho^2 \sigma_u^2) I_n$.
Furthermore, the closed-loop $L_i$ given by playing a feedback matrix $K$ that satisfies $\norm{K} \leq 1$ is:
\begin{align*}
  L_i = \rho U_i(U_i + K^\T)^\T \:.
\end{align*}
It is clear that $\norm{L_i} \leq 2 \rho$ and hence if $\rho < 1/2$ then this system is stable.
Furthermore, we have that $\rank(L_i) \leq d$.
With this, we can control:
\begin{align*}
    \Tr(\E[ x_t x_t^\T ]) &= \Tr\left(\sum_{\ell=0}^{t-1} L_i^\ell \Sigma_i (L_i^\ell)^\T\right) \leq (\sigma_w^2 + \rho^2\sigma_u^2) \sum_{\ell=0}^{t-1} \norm{L_i^\ell}_F^2 \\
    &\leq d(\sigma_w^2 + \rho^2\sigma_u^2)  \sum_{\ell=0}^{t-1} \norm{L_i^\ell}^2 \leq  \frac{ d(\sigma_w^2 + \rho^2\sigma_u^2) }{1 - (2\rho)^2} \:.
\end{align*}
Hence for one trajectory $Z = (x_0, u_0, x_1, u_1, ..., x_{T-1}, u_{T-1}, x_T)$, conditioned on a particular $K$,
\begin{align*}
    \KL(\Pr_{i|K}, \Pr_{j|K}) &\leq \sum_{t=0}^{T-1} \frac{1}{2\sigma_w^2} \E_{x_t \sim \Pr_{i|K}}[\norm{(L_i - L_j) x_t}^2] \\
  &\leq \frac{8 \rho^2}{\sigma_w^2} \sum_{t=0}^{T-1} \Tr(\E[x_tx_t^\T]) \\
  &\leq \frac{8 (\sigma_w^2 + \rho^2\sigma_u^2)\rho^2 T d}{\sigma_w^2 (1-(2\rho)^2)} \:.
\end{align*}
This allows us to bound the KL between the distributions involving all the iterations as:
\begin{align*}
    \KL(\Pr_{i}, \Pr_{j}) = \sum_{\ell=1}^{N} \E_{K_\ell \sim \Pr_i}[ \KL(\Pr_{i|K_\ell}, \Pr_{j|K_\ell}) ] \leq \frac{8 (\sigma_w^2 + \rho^2\sigma_u^2)\rho^2 N T d}{\sigma_w^2 (1-(2\rho)^2)} \:.
\end{align*}
Assuming $d(n-d)$ is greater than an absolute constant, we can set $\rho$ to be
(recall we have $N$ different rollouts):
\begin{align*}
    \rho^2 \asymp \frac{\sigma_w^2}{\sigma_w^2 + \sigma_u^2} \frac{n-d}{T N} \:,
\end{align*}
and bound $\frac{I(V;Z) + \log{2}}{\log{M}} \leq 1/2$.
The result now follows from plugging in our choice of $\rho$ into
\eqref{eq:fano}.

\section{Deferred Proofs for Asymptotic Toolbox}
\label{sec:appendix:toolbox_proofs}

Our main limit theorem is the following CLT for ergodic Markov chains.
\begin{thm}[Corollary 2 of \citet{jones04}]
\label{thm:markov_chain_CLT}
Suppose that $\{x_t\}_{t = 0}^{\infty} \subseteq X$ is a geometrically ergodic (Harris) Markov chain with stationary distribution $\pi$.
Let $f : X \rightarrow \R$ be a Borel function. Suppose that $\E_{\pi}[ \abs{f}^{2+\delta} ] < \infty$
for some $\delta > 0$. Then for any initial distribution, we have:
\begin{align*}
    \sqrt{n} \left( \frac{1}{n} \sum_{i=1}^{n} f(x_i) - \E_{\pi}[f(x)] \right) \distconv \calN(0, \sigma_f^2) \:,
\end{align*}
where
\begin{align*}
    \sigma_f^2 := \Var_{\pi}(f(x_0)) + 2 \sum_{i=1}^{\infty} \Cov_{\pi}( f(x_0), f(x_i) ) \:.
\end{align*}
\end{thm}

\subsection{Proof of Lemma~\ref{lem:ls_asymptotic_dist}}

\begin{proof}
Let $X \in \R^{T \times n}$ be the data matrix with rows $(x_0, ..., x_{T-1})$
and $W \in \R^{T \times n}$ be the noise matrix with rows $(w_0, ..., w_{T-1})$.
We write:
\begin{align*}
  \Lh(T) - \Lstar = - \lambda \Lstar (X^\T X + \lambda I_n)^{-1} + W^\T X (X^\T X + \lambda I_n)^{-1} \:.
\end{align*}
Using the fact that $\vec(AXB) = (B^\T \otimes A) \vec(X)$,
\begin{align*}
  \sqrt{T}\vec(\Lh(T) - \Lstar) = - \sqrt{T} \vec(\lambda \Lstar (X^\T X + \lambda I_n)^{-1} ) + ( (T^{-1} X^\T X)^{-1} \otimes I_n) \vec(T^{-1/2} W^\T X) \:.
\end{align*}
It is well-known that $\{x_t\}$ is geometrically ergodic (see e.g.~\citet{mokkadem88}), and therefore
the augmented Markov chain $\{ (x_t, w_t) \}$ is geometrically ergodic as well.
By Theorem~\ref{thm:markov_chain_CLT} combined with the Cram{\'{e}}r-Wold theorem we conclude:
\begin{align*}
  \vec(T^{-1/2} W^\T X) = T^{-1/2} \sum_{t=1}^{T} \vec(w_t x_t^\T) \distconv \calN(0, \E_{x \sim \nu_\infty, w}[ \vec(w x^\T)\vec(w x^\T)^\T ]) \:.
\end{align*}
Above, we let $\nu_\infty$ denote the stationary distribution of $\{x_t\}$.
We note that the cross-correlation terms disappear in the asymptotic covariance due to the martingale
difference property of $\sum_{t=0}^{T-1} w_t x_t^\T$.
We now use the identity $\vec(w x^\T) = (x \otimes I_n) w$ and compute
\begin{align*}
  \E_{x \sim \nu_\infty, w}[ \vec(w x^\T)\vec(w x^\T)^\T ] &= \E_{x \sim \nu_\infty, w}[ (x \otimes I_n) ww^\T (x^\T \otimes I_n) ] \\
  &= \sigma_w^2 \E_{x \sim \nu_\infty}[ (x \otimes I_n)(x^\T \otimes I_n) ] \\
  &= \sigma_w^2 \E_{x \sim \nu_\infty}[ (xx^\T \otimes I_n) ] \\
  &= \sigma_w^2 (P_\infty \otimes I_n) \:.
\end{align*}
We have that $T^{-1} X^\T X \asconv P_\infty$ by the ergodic theorem.
Therefore by the continuous mapping theorem followed by Slutsky's theorem,
we have that
\begin{align*}
    ( (T^{-1} X^\T X)^{-1} \otimes I_n) \vec(T^{-1/2} W^\T X) \distconv \calN(0, \sigma_w^2 (P_\infty^{-1} \otimes I_n)) \:.
\end{align*}
On the other hand, we have:
\begin{align*}
    \sqrt{T} \vec( \lambda \Lstar ( X^\T X + \lambda I_n)^{-1} ) &= \frac{1}{\sqrt{T}} \vec(\lambda \Lstar ( T^{-1} X^\T X + T^{-1} \lambda I_n )^{-1} ) \asconv 0 \:.
\end{align*}
The claim now follows by another application of Slutsky's theorem.
\end{proof}

\subsection{Proof of Lemma~\ref{lem:ls_driven_asymptotic_dist}}
\begin{proof}
Let $Z^{(i)} \in \R^{T \times (n+d)}$ be a data matrix with the rows $(z_0^{(i)}, ..., z_{T-1}^{(i)})$,
and let $W^{(i)} \in \R^{T \times n}$ be the noise matrix with the rows $(w_0^{(i)}, ..., w_{T-1}^{(i)})$.
With this notation we write:
\begin{align*}
  \Thetah(N) - \Theta_\star &= \left( \sum_{i=1}^{N} \frac{1}{T} \sum_{t=0}^{T-1} z_{t+1}^{(i)} (z_t^{(i)})^\T \right) \left( \sum_{i=1}^{N} \frac{1}{T} \sum_{t=0}^{T-1} z_t^{(i)} (z_t^{(i)})^\T + \lambda I_{n+d} \right)^{-1} - \Theta_\star \\
  &= \Theta_\star \left( \sum_{i=1}^{N} \frac{1}{T} (Z^{(i)})^\T Z^{(i)} \right) \left( \sum_{i=1}^{N} \frac{1}{T} (Z^{(i)})^\T Z^{(i)} + \lambda I_{n+d} \right)^{-1} - \Theta_\star \\
  &\qquad+ \left( \sum_{i=1}^{N} \frac{1}{T} (W^{(i)})^\T Z^{(i)} \right)\left( \sum_{i=1}^{N} \frac{1}{T} (Z^{(i)})^\T Z^{(i)} + \lambda I_{n+d} \right)^{-1} \\
  &= - \lambda \Theta_\star \left( \sum_{i=1}^{N} \frac{1}{T} (Z^{(i)})^\T Z^{(i)} + \lambda I_{n+d} \right)^{-1} \\
  &\qquad + \left( \sum_{i=1}^{N} \frac{1}{T} (W^{(i)})^\T Z^{(i)} \right)\left( \sum_{i=1}^{N} \frac{1}{T} (Z^{(i)})^\T Z^{(i)} + \lambda I_{n+d} \right)^{-1} \\
  &=: G_1(N) + G_2(N) \:.
\end{align*}
Taking vec of $G_2(N)$:
\begin{align*}
  \vec(G_2(N)) &= \left( \left( \frac{1}{N} \sum_{i=1}^{N} \frac{1}{T} (Z^{(i)})^\T Z^{(i)} + \frac{\lambda}{N} I_{n+d} \right)^{-1} \otimes I_n \right) \vec\left( \frac{1}{N} \sum_{i=1}^{N} \frac{1}{T} \sum_{t=0}^{T-1} w_t^{(i)} (z_t^{(i)})^\T \right) \:.
\end{align*}
Now we write $\vec( w_t z_t^\T) = (z_t \otimes I_n) w_t$ and hence
\begin{align*}
  \E\left[\vec\left( \frac{1}{T} \sum_{t=0}^{T-1} w_t z_t^\T \right) \vec\left( \frac{1}{T} \sum_{t=0}^{T-1} w_t z_t^\T \right)^\T \right] &= \frac{1}{T^2} \sum_{t_1, t_2 = 0}^{T-1} \E[(z_{t_1} \otimes I_n) w_{t_1} w_{t_2}^\T (z_{t_2}^\T \otimes I_n)] \\
  &=  \frac{\sigma_w^2}{T^2} \sum_{t=0}^{T-1} \E[z_tz_t^\T] \otimes I_n \:.
\end{align*}
We have that:
\begin{align*}
  \frac{1}{N} \sum_{i=1}^{N} \frac{1}{T} (Z^{(i)})^\T Z^{(i)} + \frac{\lambda}{N} I_{n+d} \asconv \frac{1}{T} \sum_{t=0}^{T-1} \E[z_t z_t^\T] \:.
\end{align*}
Hence by the central limit theorem combined with the continuous mapping theorem and Slutsky's theorem,
\begin{align*}
  \sqrt{N} \vec(G_1(N)) &\asconv 0 \:, \\
  \sqrt{N} \vec(G_2(N)) &\distconv \calN\left(0, \frac{\sigma_w^2}{T} \left[ \frac{1}{T} \sum_{t=0}^{T-1} \E[z_tz_t^\T] \right]^{-1} \otimes I_n \right) \\
  &= \calN\left(0, \frac{\sigma_w^2}{T} \bmattwo{[\frac{1}{T}\sum_{t=0}^{T-1} \E[x_tx_t^\T] ]^{-1}}{0}{0}{(1/\sigma_u^2) I_d} \otimes I_n \right) \:.
\end{align*}
To finish the proof, we note that $\E[ x_tx_t^\T ] = \sum_{\ell=0}^{t-1} \Astar^\ell M (\Astar^\ell)^\T := P_t$ with $M := \sigma_u^2 \Bstar\Bstar^\T + \sigma_w^2 I_n$ and $P_0 = 0$ (since $x_0 = 0$).
Since $\Astar$ is stable, there exists a $\rho \in (0, 1)$ and $C > 0$ such that $\norm{\Astar^k} \leq C \rho^k$ for all $k \geq 0$.
Hence,
\begin{align*}
    \norm{P_\infty - P_t} &= \bignorm{ \sum_{\ell=t}^{\infty} \Astar^\ell M (\Astar^\ell)^\T } \leq C^2 \norm{M} \sum_{\ell=t}^{\infty} \rho^{2\ell} = C^2 \norm{M} \frac{\rho^{2t}}{1-\rho^2} \:.
\end{align*}
Therefore,
\begin{align*}
    \bignorm{\frac{1}{T} \sum_{t=0}^{T-1} P_t - P_\infty} &= \bignorm{\frac{1}{T} \sum_{t=1}^{T-1} (P_t - P_\infty) + \frac{1}{T} P_\infty} \\
    &\leq \frac{1}{T} \sum_{t=1}^{T-1} \norm{P_\infty - P_t} + \frac{1}{T} \norm{P_\infty} \\
    &\leq \frac{C^2 \norm{M}}{T (1-\rho^2)} \sum_{t=1}^{T-1} \rho^{2t} + \frac{1}{T} \norm{P_\infty} \\
    &\leq \frac{C^2 \norm{M}}{T (1-\rho^2)^2} + \frac{1}{T} \norm{P_\infty} = O(1/T) \:.
\end{align*}
Therefore, $[\frac{1}{T}\sum_{t=0}^{T-1} \E[x_tx_t^\T] ]^{-1} = P_\infty^{-1} + O(1/T)$ from which the claim follows.
\end{proof}

\subsection{Proof of Lemma~\ref{lem:lstd_asymptotic_distribution}}

\begin{proof}
Let $c_t = x_t^\T (Q + K^\T R K) x_t$.
From Bellman's equation, we have $c_t - \lambda_\star = (\phi(x_t) - \psi(x_t))^\T w_\star$.
We write:
\begin{align*}
    \wlstd(T) - w_\star &= \left( \sum_{t=0}^{T-1} \phi(x_t) (\phi(x_t) - \phi(x_{t+1}))^\T \right)^{-1} \left( \sum_{t=0}^{T-1} (c_t - \lambda_\star) \phi(x_t) \right) - w_\star \\
    &= \left( \sum_{t=0}^{T-1} \phi(x_t) (\phi(x_t) - \phi(x_{t+1}))^\T \right)^{-1} \left( \sum_{t=0}^{T-1} \phi(x_t) (\phi(x_t) - \psi(x_t))^\T  \right) w_\star - w_\star \\
    &= \left( \sum_{t=0}^{T-1} \phi(x_t) (\phi(x_t) - \phi(x_{t+1}))^\T \right)^{-1} \left( \sum_{t=0}^{T-1} \phi(x_t) (\phi(x_{t+1}) - \psi(x_t))^\T w_\star \right) \\
    &= \left( \frac{1}{T} \sum_{t=0}^{T-1} \phi(x_t) (\phi(x_t) - \phi(x_{t+1}))^\T \right)^{-1} \left( \frac{1}{T} \sum_{t=0}^{T-1} \phi(x_t) (\phi(x_{t+1}) - \psi(x_t))^\T w_\star \right) \:.
\end{align*}
We now proceed by considering the Markov chain $\{z_t := (x_t, w_t)\}$.
Observe that $x_{t+1}$ is $z_t$-measurable, and furthermore
the stationary distribution of this chain is $\nu_\infty \times \calN(0, \sigma_w^2 I_n)$.
From this we conclude two things. First, we conclude by the ergodic theorem that the term inside the
inverse converges a.s.\ to $A_\infty$  and hence the inverse converges a.s.\ to
$A_\infty^{-1}$ by the continuous mapping theorem.
Next, Theorem~\ref{thm:markov_chain_CLT} combined with the
Cram{\'e}r-Wold theorem allows us to conclude that
\begin{align*}
    \frac{1}{\sqrt{T}} \sum_{t=1}^{T}\phi(x_t) (\phi(x_{t+1}) - \psi(x_t))^\T w_\star \distconv \calN(0, B_\infty) \:.
\end{align*}
The final claim now follows by Slutsky's theorem.
\end{proof}

\subsection{Proof of Corollary~\ref{cor:a_infty_b_infty}}

\begin{proof}
In the proof we write $\Sigma = \sigma_w^2 I_n$.
First, we note that a quick computation shows that $\psi(x) = \svec(Lxx^\T L^\T + \Sigma)$.

\paragraph{Matrix $A_\infty$.}
We have
\begin{align*}
    \phi(x) - \phi(x') &= \svec(xx^\T - (L x + w)(Lx + w)^\T)  \\
    &= \svec( xx^\T - Lxx^\T L^\T - L x w^\T - w x^\T L^\T - w w^\T ) \:.
\end{align*}
Hence, conditioning on $x$ and iterating expectations, we have
\begin{align*}
    A_\infty = \E_{x \sim \nu_\infty}[  \phi(x) \svec( xx^\T - Lxx^\T L^\T - \Sigma )^\T ] \:.
\end{align*}
Now let $m, n$ be two test vectors and $M = \smat(m), N = \smat(n)$.
We have that,
\begin{align*}
    m^\T A_\infty n &= \E_{x \sim \nu_\infty}[ x^\T M x \ip{xx^\T - L xx^\T L^\T - \Sigma}{N} ] \\
    &= \E_{x \sim \nu_\infty}[ x^\T M x (x^\T (N - L^\T N L ) x - \ip{\Sigma}{N} ) ] \\
    &= \E_{x \sim \nu_\infty}[ x^\T M x x^\T (N - L^\T N L) x ] - \ip{\Sigma}{N} \E_{x \sim \nu_\infty}[x^\T M x] \\
    &= \E_{g}[ g^\T P_\infty^{1/2} M P_\infty^{1/2} g g^\T P_\infty^{1/2} (N - L^\T N L) P_\infty^{1/2} g ] - \ip{\Sigma}{N} \ip{M}{P_\infty} \\
    &= 2 \ip{P_\infty^{1/2} M P_\infty^{1/2}}{P_\infty^{1/2}(N - L^\T N L ) P_\infty^{1/2}} + \ip{M}{P_\infty} \ip{N - L^\T N L}{P_\infty} - \ip{\Sigma}{N} \ip{M}{P_\infty} \\
    &= 2 \ip{P_\infty^{1/2} M P_\infty^{1/2}}{P_\infty^{1/2}(N - L^\T N L ) P_\infty^{1/2}} \:,
\end{align*}
where the last identity follows since $L P_\infty L^\T - P_\infty + \Sigma = 0$.
We therefore have:
\begin{align*}
    A_\infty &= (P_\infty \otimes_s P_\infty) - (P_\infty L^\T \otimes_s P_\infty L^\T) \\
    &= (P_\infty \otimes_s P_\infty) ( I - L^\T \otimes_s L^\T ) \:.
\end{align*}
Note that this writes $A_\infty$ as the product of two invertible matrices and hence $A_\infty$ is invertible.

\paragraph{Matrix $B_\infty$.}
We have
\begin{align*}
    \ip{\phi(x') - \psi(x)}{w_\star} &= \svec( L x w^\T + w x^\T L^\T + w w^\T - \Sigma )^\T w_\star \\
    &= 2 x^\T L^\T P_\star w + \ip{ww^\T - \Sigma}{P_\star} \:.
\end{align*}
Hence,
\begin{align*}
    \ip{\phi(x') - \psi(x)}{w_\star}^2 &= 4 (x^\T L^\T P_\star w)^2 + \ip{ww^\T - \Sigma}{P_\star}^2 + 4 x^\T L^\T P_\star w  \ip{ww^\T - \Sigma}{P_\star} \\
    &=: T_1 + T_2 + T_3 \:.
\end{align*}
Now we have that $m^\T B_\infty n$ is
\begin{align}
  m^\T B_\infty n &= \E[ T_1 x^\T M x x^\T N x ] + \E[ T_2 x^\T M x x^\T N x ] + \E[ T_3 x^\T M x x^\T N x ] \:. \label{eq:b_matrix_decomp}
\end{align}
First, we have
\begin{align*}
    \E[ T_1 x^\T M x x^\T N x ] &= 4 \E[ (x^\T L^\T P_\star w)^2 x^\T M x x^\T N x ] \\
    &= 4 \E[ x^\T L^\T P_\star w w^\T P_\star L x x^\T M x x^\T N x ] \\
    &= 4 \E[ x^\T L^\T P_\star \Sigma P_\star L x x^\T M x x^\T N x ] \\
    &= 4 \E_{g}[  g^\T (P_\infty^{1/2} L^\T P_\star \Sigma P_\star L P_\infty^{1/2}) g g^\T (P_\infty^{1/2} M P_\infty^{1/2} ) g  g^\T (P_\infty^{1/2} N P_\infty^{1/2} ) g]
\end{align*}
Now we state a result from Magnus to compute the expectation of the product of three
quadratic forms of Gaussians.
\begin{lem}[See e.g.~\citet{magnus79}]
\label{lem:sixth_moment}
Let $g \sim \calN(0, I)$ and $A_1,A_2,A_3$ be symmetric matrices. Then,
\begin{align*}
    &\E[ g^\T A_1 g g^\T A_2 g g^\T A_3 g ] = \Tr(A_1)\Tr(A_2)\Tr(A_3) \\
    &\qquad+ 2 ( \Tr(A_1)\Tr(A_2A_3) + \Tr(A_2)\Tr(A_1 A_3) + \Tr(A_3) \Tr(A_1 A_2) ) \\
    &\qquad+ 8 \Tr(A_1 A_2 A_3)  \:.
\end{align*}
\end{lem}
Now by setting
\begin{align*}
    A_1 &= P_\infty^{1/2} L^\T P_\star \Sigma P_\star L P_\infty^{1/2} \:, \\
    A_2 &= P_\infty^{1/2} M P_\infty^{1/2} \:, \\
    A_3 &= P_\infty^{1/2} N P_\infty^{1/2} \:,
\end{align*}
we can compute the expectation $\E[ T_1 x^\T M x x^\T N x]$ using Lemma~\ref{lem:sixth_moment}.
In particular,
\begin{align*}
    \Tr(A_1)\Tr(A_2)\Tr(A_3) &= \ip{P_\infty}{L^\T P_\star \Sigma P_\star L} m^\T \svec(P_\infty) \svec(P_\infty)^\T n \:, \\
    \Tr(A_1)\Tr(A_2A_3) &= \ip{P_\infty}{L^\T P_\star \Sigma P_\star L} m^\T (P_\infty \otimes_s P_\infty) n \:, \\
    \Tr(A_2)\Tr(A_1A_3) &= m^\T \svec(P_\infty) \svec(P_\infty L^\T P_\star \Sigma P_\star L P_\infty)^\T n \:, \\
    \Tr(A_3)\Tr(A_1A_2) &= m^\T \svec(P_\infty L^\T P_\star \Sigma P_\star L P_\infty)\svec(P_\infty)^\T n \:, \\
    \Tr(A_1A_2A_3) &= m^\T (P_\infty L^\T P_\star \Sigma P_\star L P_\infty \otimes_s P_\infty) n \:.
\end{align*}
Hence,
\begin{align*}
    &\E[ g^\T A_1 g g^\T A_2 g g^\T A_3 g ] \\
    &\qquad= m^\T(\ip{P_\infty}{L^\T P_\star \Sigma P_\star L} (2 (P_\infty \otimes_s P_\infty) +  \svec(P_\infty) \svec(P_\infty)^\T) \\
    &\qquad\qquad + 2\svec(P_\infty) \svec(P_\infty L^\T P_\star \Sigma P_\star L P_\infty)^\T + 2\svec(P_\infty L^\T P_\star \Sigma P_\star L P_\infty)\svec(P_\infty)^\T \\
    &\qquad\qquad + 8 (P_\infty L^\T P_\star \Sigma P_\star L P_\infty \otimes_s P_\infty))n
\end{align*}
Next, we compute
\begin{align*}
    \E[T_2 x^\T M x x^\T N x] &= \E[ \ip{ww^\T - \Sigma}{P_\star}^2 x^\T M x x^\T N x] \\
    &= \E[ \ip{ww^\T - \Sigma}{P_\star}^2 ] \E[ x^\T M x x^\T N x ] \:.
\end{align*}
First, we have
\begin{align*}
    \E[ \ip{ww^\T - \Sigma}{P_\star}^2 ] &= \E[ (w^\T P_\star w)^2 ] - 2 \ip{\Sigma}{P_\star} \E[ w^\T P_\star w ] + \ip{\Sigma}{P_\star}^2 \\
    &= 2 \norm{\Sigma^{1/2} P_\star \Sigma^{1/2} }_F^2 + \ip{P_\star}{\Sigma}^2 - 2 \ip{\Sigma}{P_\star}^2 + \ip{P_\star}{\Sigma}^2 \\
    &= 2 \norm{\Sigma^{1/2} P_\star \Sigma^{1/2} }_F^2 \:.
\end{align*}
On the other hand,
\begin{align*}
    \E[ x^\T M x x^\T N x ] = 2 \ip{P_\infty^{1/2} M P_\infty^{1/2}}{P_\infty^{1/2} N P_\infty^{1/2}} + \ip{M}{P_\infty} \ip{N}{P_\infty} \:.
\end{align*}
Combining these calculations,
\begin{align*}
    \E[T_2 x^\T M x x^\T N x] &=2 \norm{\Sigma^{1/2} P_\star \Sigma^{1/2} }_F^2 (  2 \ip{P_\infty^{1/2} M P_\infty^{1/2}}{P_\infty^{1/2} N P_\infty^{1/2}} + \ip{M}{P_\infty} \ip{N}{P_\infty} ) \\
    &= 2 \norm{\Sigma^{1/2} P_\star \Sigma^{1/2} }_F^2 m^\T (2 (P_\infty \otimes_s P_\infty)  + \svec(P_\infty) \svec(P_\infty)^\T) n
\end{align*}
Finally, we have $\E[ T_3 x^\T M x x^\T N x ] = 0$, which is easy to see because it involves odd powers of $w$.
This gives us that $B_\infty$ is:
\begin{align*}
    B_\infty &= (\ip{P_\infty}{L^\T P_\star \Sigma P_\star L} + 2 \norm{\Sigma^{1/2} P_\star \Sigma^{1/2} }_F^2  ) (2 (P_\infty \otimes_s P_\infty) +  \svec(P_\infty) \svec(P_\infty)^\T) \\
    &\qquad + 2\svec(P_\infty) \svec(P_\infty L^\T P_\star \Sigma P_\star L P_\infty)^\T + 2\svec(P_\infty L^\T P_\star \Sigma P_\star L P_\infty)\svec(P_\infty)^\T \\
    &\qquad + 8 (P_\infty L^\T P_\star \Sigma P_\star L P_\infty \otimes_s P_\infty) \:.
\end{align*}
This completes the proof of the formulas for $A_\infty$ and $B_\infty$.

To obtain the lower bound,
we need the following lemma which gives a useful lower bound
to Lemma~\ref{lem:sixth_moment}.
\begin{lem}
\label{lem:sixth_moment_lb}
Let $A_1$ be positive semi-definite and let $A_2$ be symmetric. Let $g \sim \calN(0, I)$.
We have that:
\begin{align*}
  \E[ g^\T A_1 g (g^\T A_2 g)^2 ] &\geq 2 \Tr(A_1)\Tr(A_2^2) + 4 \Tr(A_1 A_2^2) \:.
\end{align*}
\end{lem}
\begin{proof}
Suppose that $A_1 \neq 0$, otherwise the bound holds vacuously.
From Lemma~\ref{lem:sixth_moment},
\begin{align*}
  \E[ g^\T A_1 g (g^\T A_2 g)^2 ] &= \Tr(A_1)\Tr(A_2)^2 + 2\Tr(A_1)\Tr(A_2^2) + 4\Tr(A_2)\Tr(A_1A_2) + 8\Tr(A_1A_2^2) \:.
\end{align*}
Since $A_1$ is PSD and non-zero, this means that $\Tr(A_1) > 0$.
We proceed as follows:
\begin{align*}
  4 \abs{\Tr(A_2) \Tr(A_1A_2)} &= 2 \abs{\Tr(A_2)\Tr(A_1)^{1/2}} \bigabs{2\frac{\Tr(A_1A_2)}{\Tr(A_1)^{1/2}}} \\
  &\stackrel{(a)}{\leq} \Tr(A_1)\Tr(A_2)^2 + 4 \frac{\Tr(A_1A_2)^2}{\Tr(A_1)} \\
  &= \Tr(A_1)\Tr(A_2)^2 + 4 \frac{\Tr(A_1^{1/2} A_1^{1/2}A_2)^2}{\Tr(A_1)} \\
  &\stackrel{(b)}{\leq} \Tr(A_1)\Tr(A_2)^2 + 4 \frac{\norm{A_1^{1/2}}_F^2 \norm{A_1^{1/2} A_2}_F^2}{\Tr(A_1)} \\
  &= \Tr(A_1)\Tr(A_2)^2 + 4 \Tr(A_1 A_2^2) \:,
\end{align*}
where in (a) we used Young's inequality and
in (b) we used Cauchy-Schwarz.
The claim now follows.
\end{proof}

We now start from the decomposition \eqref{eq:b_matrix_decomp} for $B_\infty$,
with $m = n$ and noting that $\E[ T_2 (x^\T M x)^2 ] \geq 0$ and $\E[ T_3 (x^\T M x)^3 ] = 0$:
\begin{align*}
  m^\T B_\infty m &\geq \E[ T_1 (x^\T M x)^2 ] \\
  &\stackrel{(a)}{\geq} 8 \ip{P_\infty}{L^\T P_\star \Sigma P_\star L} m^\T (P_\infty \otimes_s P_\infty) m + 16 m^\T (P_\infty L^\T P_\star \Sigma P_\star L P_\infty \otimes_s P_\infty) m \:.
\end{align*}
Above in (a) we applied the lower bound from Lemma~\ref{lem:sixth_moment_lb}.
Hence since $m$ is arbitrary,
\begin{align*}
  B_\infty \succeq 8\ip{P_\infty}{L^\T P_\star \Sigma P_\star L}(P_\infty \otimes_s P_\infty) + 16 (P_\infty L^\T P_\star \Sigma P_\star L P_\infty \otimes_s P_\infty) \:.
\end{align*}
We also have that $A_\infty = (P_\infty \otimes_s P_\infty) (I - L^\T \otimes L^\T)$, and hence
$A_\infty^{-1} = (I - L^\T \otimes L^\T)^{-1} (P_\infty^{-1} \otimes_s P_\infty^{-1})$.
Therefore,
\begin{align*}
  A_\infty^{-1} B_\infty A_{\infty}^{-\T} &\succeq 8 \ip{P_\infty}{L^\T P_\star \Sigma P_\star L} (I - L^\T \otimes_s L^\T)^{-1} (P_\infty^{-1} \otimes_s P_\infty^{-1})(I - L^\T \otimes_s L^\T)^{-\T} \\
  &\qquad + 16(I - L^\T \otimes_s L^\T)^{-1} (L^\T P_\star \Sigma P_\star L \otimes_s P_\infty^{-1})(I - L^\T \otimes_s L^\T)^{-\T} \:.
\end{align*}

\end{proof}

\subsection{Proof of Lemma~\ref{lem:moment_control_ls}}
\begin{proof}
Recall in the notation of the proof of Lemma~\ref{lem:ls_asymptotic_dist},
\begin{align*}
  \Lh(T) - \Lstar = - \lambda \Lstar (X^\T X + \lambda I_n)^{-1} + W^\T X (X^\T X + \lambda I_n)^{-1} \:.
\end{align*}
Now let us suppose that we are on an event where $X^\T X$ is invertible.
Let $X = U \Sigma V^\T$ denote the compact SVD of $X$.
We have:
\begin{align*}
  \norm{\Lh(T) - \Lstar} &\leq \lambda \frac{\norm{\Lstar}}{\lambda_{\min}(X^\T X + \lambda I_n)} + \norm{ W^\T X (X^\T X + \lambda I_n)^{-1} }  \\
  &\stackrel{(a)}{\leq} \lambda \frac{\norm{\Lstar}}{\lambda_{\min}(X^\T X + \lambda I_n)} + \norm{ W^\T X (X^\T X)^{-1} } \:.
\end{align*}
The inequality (a) holds due to the following.
First observe that
$(X^\T X + \lambda I_n)^{-2} \preceq (X^\T X)^{-2}$. Therefore with $M = W^\T X$,
conjugating both sides by $M$,
we have $M (X^\T X + \lambda I_n)^{-2} M^\T \preceq M (X^\T X)^{-2} M^\T$. Hence,
\begin{align*}
  \norm{M (X^\T X + \lambda I_n)^{-1}} &= \sqrt{ \lambda_{\max}( M (X^\T X + \lambda I_n)^{-2} M^\T ) } \\
  &\leq \sqrt{\lambda_{\max}( M (X^\T X)^{-2} M^\T ) } \\
  &= \norm{M (X^\T X)^{-1}} \:.
\end{align*}
By Theorem 2.4 of~\citet{simchowitz18}
for $T \geq C_{\Lstar, n} \log(1/\delta)$,
there exists an event $\calE$ with $\Pr(\calE) \geq 1-\delta$
such that on $\calE$ we have:
\begin{align*}
  \norm{\Lh_{\mathrm{ols}}(T) - \Lstar} \leq C'_{\Lstar, n} \sqrt{\log(1/\delta)/T} \:, \:\: X^\T X \succeq C''_{\Lstar, n} T \cdot I_n \:.
\end{align*}
Hence on this event we have $\norm{\Lh(T) - \Lstar} \leq C'_{\Lstar,n,\lambda} \sqrt{\log(1/\delta)/T}$.

For the remainder of the proof, $O(\cdot)$ will hide
constants that depend on $\Lstar, n, p, \lambda$ but not on $T$ or $\delta$.
We bound the $p$-th moment as follows. We decompose:
\begin{align*}
  \E[ \norm{\Lh(T) - \Lstar}^p ] = \E[ \norm{\Lh(T) - \Lstar}^p \ind_{\calE} ] + \E[ \norm{\Lh(T) - \Lstar}^p \ind_{\calE^c} ] \:.
\end{align*}
On $\calE$ we have by the inequality $(a+b)^p \leq 2^{p-1} (a^p + b^p)$ for non-negative $a,b$,
\begin{align*}
  \norm{\Lh(T) - \Lstar}^p \leq 2^{p-1}( O(\lambda^p/T^p) + O((\log(1/\delta) / T)^{p/2}) ) \:.
\end{align*}
On the other hand, we always have:
\begin{align*}
  \norm{\Lh(T) - \Lstar}^p \leq 2^{p-1}( \norm{\Lstar}^p + (\norm{W^\T X} / \lambda)^p ) \:.
\end{align*}
Hence:
\begin{align*}
  \E[ \norm{\Lh(T) - \Lstar}^p \ind_{\calE^c} ] &\leq 2^{p-1} \norm{\Lstar}^p \Pr(\calE^c) + \frac{2^{p-1}}{\lambda^p} \E[ \norm{W^\T X}^p \ind_{\calE^c} ] \\
  &\leq 2^{p-1} \norm{\Lstar}^p  \delta + \frac{2^{p-1}}{\lambda^p} \sqrt{ \E[\norm{W^\T X}^{2p}] \delta } \:.
\end{align*}
We will now compute a very crude bound on $\E[ \norm{W^\T X}^{2p} ]$ which will suffice.
For non-negative $a_t$, we have $(a_1 + ... + a_T)^{2p} \leq T^{2p-1} (\sum_{t=1}^{T} a_i^{2p})$ by H{\"{o}}lder's inequality.
Hence
\begin{align*}
  \E[ \norm{W^\T X}^{2p} ] &= \E\left[ \bignorm{\sum_{t=0}^{T-1} w_i x_i^\T }^{2p} \right] \\
  &\leq T^{2p-1} \E\left[ \sum_{t=1}^{T} \norm{w_t}^{2p} \norm{x_t}^{2p}    \right] \\
  &= T^{2p-1} \E[ \norm{w_1}^{2p} ]\sum_{t=1}^{T} \E[\norm{x_t}^{2p}] \\
  &\leq T^{2p} \E[ \norm{w_1}^{2p} ] \norm{P_\infty}^{p} \E_{g \sim \calN(0, I)}[ \norm{g}^{2p} ] \\
  &= O(T^{2p}) \:.
\end{align*}
Above, $P_\infty$ denotes the covariance of the stationary distribution of
$\{x_t\}$.
Continuing from above:
\begin{align*}
  \E[ \norm{\Lh(T) - \Lstar}^p \ind_{\calE^c} ] &=2^{p-1} \norm{\Lstar}^p  \delta + \frac{2^{p-1}}{\lambda^p} \sqrt{ O(T^{2p}) \delta } \:.
\end{align*}
We now set $\delta = O(1/T^{3p})$ so that the term above is $O(1/T^{p/2})$.
Doing this we obtain that for $T$ sufficiently large (as a function of
only $\Lstar, p, \lambda$),
\begin{align*}
  \E[ \norm{\Lh(T) - \Lstar}^p ] \leq O(1/T^{p/2}) \:.
\end{align*}
\end{proof}

\section{Proof of Lemma~\ref{lem:sgd_asymptotics}}
\label{sec:proof_sgd_asymptotics}

We now state a high probability bound for SGD.
This is a straightforward modification of Lemma 6 from~\citet{rakhlin12}
(modifications are needed to deal with the lack of almost surely bounded gradients),
and hence we omit the proof.
\begin{lem}[Lemma 6, \citet{rakhlin12}]
\label{lem:sgd_whp}
Let the assumptions of Lemma~\ref{lem:sgd_asymptotics} hold.
Define two constants:
\begin{align*}
    M := \sup_{\theta \in \Theta} \norm{\theta}_2 \:, \:\: G_3 := \sup_{\theta \in \Theta} \norm{\nabla F(\theta)}_2 \:.
\end{align*}
Note that since $\Theta$ is compact, both $M$ and $G_3$ are finite.
Fix a $T \geq 4$ and $\delta \in (0, 1/e)$. We have that with probability at least $1-\delta$, for all $t \leq T$,
\begin{align*}
    \norm{\theta_t - \theta_\star}_2^2 \lesssim \frac{\polylog(T/\delta)}{t} \left( \frac{G_1^2+G_2^2}{m^2} + \frac{M(G_2 + G_3)}{m} \right) \:.
\end{align*}
\end{lem}

We are now in a position to analyze the asymptotic variance of SGD with projection.
As mentioned previously, our argument follows closely that of \citet{toulis17}.
For the remainder of the proof, $O(\cdot)$ and $\Omega(\cdot)$ will hide all constants
except those depending on $t$ and $\delta$.
Introduce the notation:
\begin{align*}
    \tilde{\theta}_{t+1} &= \theta_t - \alpha_t g(\theta_t; \xi_t) \:, \\
    \theta_{t+1} &= \mathsf{Proj}_\Theta(\tilde{\theta}_{t+1}) \:.
\end{align*}
Let $\calE_t := \{ \tilde{\theta}_t = \theta_t \}$ be the event that
the projection step is inactive at time $t$.
Recall that we assumed that $\theta_\star$ is in the interior of $\Theta$.
This means there exists a radius $R > 0$ such that
$\{ \theta : \norm{\theta - \theta_\star}_2 \leq R \} \subseteq \Theta$.
Therefore, the event $\{ \norm{\tilde{\theta}_t - \theta_\star}_2 \leq R \} \subseteq \calE_t$.
We now decompose,
\begin{align*}
  \Var(\theta_{t+1}) &= \Var(\theta_{t+1} - \thetatilde_{t+1} + \thetatilde_{t+1}) \\
  &= \Var(\thetatilde_{t+1}) + \Var(\theta_{t+1} - \thetatilde_{t+1}) + \Cov(\theta_{t+1} - \thetatilde_{t+1}, \thetatilde_{t+1}) + \Cov(\thetatilde_{t+1}, \theta_{t+1} - \thetatilde_{t+1}) \:.
\end{align*}
We have that,
\begin{align*}
  \theta_{t+1} - \thetatilde_{t+1} &= (\theta_{t+1} - \thetatilde_{t+1}) \ind_{\calE_{t+1}^c} \:.
\end{align*}
Hence,
\begin{align*}
  \norm{\Var( \theta_{t+1} - \thetatilde_{t+1} )} &\leq \E[ \norm{ \thetatilde_{t+1} \ind_{\calE_{t+1}^c} - \theta_{t+1} \ind_{\calE_{t+1}^c} }_2^2 ] \\
  &\leq 2 ( \E[ \norm{\thetatilde_{t+1}}_2^2 \ind_{\calE_{t+1}^c} ] + \E[ \norm{\theta_{t+1}}_2^2 \ind_{\calE_{t+1}^c} ] ) \\
  &\leq 2 (\sqrt{\E[ \norm{\thetatilde_{t+1}}_2^4 ] \E[\ind_{\calE_{t+1}^c}]}  + M^2 \E[ \ind_{\calE_{t+1}^c} ] ) \:.
\end{align*}
We can bound $\E[ \norm{\thetatilde_{t+1}}_2^4 ]$ by a constant for all $t$ using
our assumption \eqref{eq:sgd_grad_momoment_bound}.
On the other hand,
\begin{align*}
  \E[ \ind_{\calE_{t+1}^c} ] \leq \Pr( \norm{\thetatilde_{t+1} - \theta_\star}_2 > R ) \:.
\end{align*}
By triangle inequality,
\begin{align*}
    \norm{\thetatilde_{t+1} - \theta_\star}_2 \leq \norm{\theta_t - \theta_\star}_2 + \alpha_t \norm{g_t}_2 \:.
\end{align*}
By Lemma~\ref{lem:sgd_whp} and the concentration bound on $\norm{g_t}_2$ from
our assumption \eqref{eq:sgd_grad_conc}, with probability at least $1-\delta$,
\begin{align*}
    \norm{\thetatilde_{t+1} - \theta_\star}_2 \leq O(\polylog(t/\delta)/\sqrt{t}) \:.
\end{align*}
Hence for $t$ large enough, $\E[ \ind_{\calE_{t+1}^c} ] \leq O( \exp( - t^{\alpha} ) )$ for some $\alpha > 0$.
This shows that $\norm{\Var( \theta_{t+1} - \thetatilde_{t+1} )} \leq O(\exp(-t^{\alpha}))$.
Similar arguments show that $\max\{ \norm{ \Cov(\theta_{t+1} - \thetatilde_{t+1}, \thetatilde_{t+1}) }, \norm{\Cov( \thetatilde_{t+1}, \theta_{t+1} - \thetatilde_{t+1})} \} \leq O(\exp(-t^{\alpha}))$.
Hence:
\begin{align*}
    \Var(\theta_{t+1}) = \Var(\thetatilde_{t+1}) + O(\exp(-t^{\alpha})) \:.
\end{align*}
Therefore,
\begin{align}
  \Var(\theta_{t+1}) &= \Var(\thetatilde_{t+1}) + O(\exp(-t^{\alpha})) \nonumber \\
  &= \Var(\theta_t - \alpha_t g(\theta_t; \xi_t)) + O(\exp(-t^{\alpha})) \nonumber \\
  &= \Var(\theta_t) + \alpha_t^2 \Var(g(\theta_t; \xi_t)) - \alpha_t \Cov(\theta_t, g(\theta_t; \xi_t)) - \alpha_t \Cov(g(\theta_t; \xi_t), \theta_t) + O(\exp(-t^{\alpha})) \nonumber \\
  &= \Var(\theta_t) + \alpha_t^2 \Var(g(\theta_t; \xi_t)) - \alpha_t \Cov(\theta_t, \nabla F(\theta_t)) - \alpha_t \Cov(\nabla F(\theta_t), \theta_t) + O(\exp(-t^{\alpha})) \label{eq:var_recur} \:.
\end{align}
Now we write:
\begin{align*}
    \Var( g(\theta_t; \xi_t) ) &= \Var( g(\theta_\star; \xi_t) + (g(\theta_t; \xi_t) - g(\theta_\star; \xi_t)) ) \\
    &= \Var( g(\theta_\star; \xi_t)) + \Var( g(\theta_t; \xi_t) - g(\theta_\star; \xi_t)) \\
    &\qquad+ \Cov( g(\theta_\star; \xi_t), g(\theta_t; \xi_t) - g(\theta_\star; \xi_t)) + \Cov(g(\theta_t; \xi_t) - g(\theta_\star; \xi_t), g(\theta_\star; \xi_t)) \:.
\end{align*}
We have by our assumption \eqref{eq:sgd_grad_lip},
\begin{align*}
    \norm{ \Var( g(\theta_t; \xi_t) - g(\theta_\star; \xi_t) ) } &\leq \E[ \norm{  g(\theta_t; \xi_t) - g(\theta_\star; \xi_t)  }^2_2 ] \\
    &= \E_{\theta_t} \E_{\xi}[ \norm{  g(\theta_t; \xi_t) - g(\theta_\star; \xi_t)  }^2_2 ] \\
    &\leq L \E[ \norm{\theta_t - \theta_\star}^2_2 ] \:.
\end{align*}
On the other hand,
\begin{align*}
    \norm{\Cov(g(\theta_\star; \xi_t), g(\theta_t; \xi_t) - g(\theta_\star; \xi_t))} &\leq 2 \E[ \norm{g(\theta_\star; \xi_t)}_2 \norm{g(\theta_t; \xi_t) - g(\theta_\star; \xi_t)}_2 ] \\
    &\leq 2 \sqrt{ \E[ \norm{g(\theta_\star; \xi_t)}_2^2 ] \E[ \norm{g(\theta_t; \xi_t) - g(\theta_\star; \xi_t)}_2^2 ] } \\
    &\leq 2 \sqrt{ L G_1^2 \E[ \norm{\theta_t - \theta_\star}_2^2 ] } \:.
\end{align*}
The same bound also holds for
$\norm{\Cov(g(\theta_t; \xi_t) - g(\theta_\star; \xi_t), g(\theta_\star; \xi_t))}$.
Since we know that $\E[ \norm{\theta_t - \theta_\star}_2^2 ] \leq O(1/t)$, this shows that:
\begin{align*}
    \Var(g(\theta_t; \xi_t)) = \Var(g(\theta_\star; \xi)) + o_t(1) \:.
\end{align*}
Next, by a Taylor expansion of $\nabla F(\theta_t)$ around $\theta_\star$, we have that:
\begin{align*}
    \nabla F(\theta_t) = \nabla^2 F(\theta_\star) (\theta_t - \theta_\star) + \mathsf{Rem}(\theta_t - \theta_\star) \:,
\end{align*}
where $\norm{\mathsf{Rem}(\theta_t - \theta_\star)} \leq O( \norm{\theta_t - \theta_\star}_2^2 )$.
Therefore, utilizing the fact that adding a non-random vector does not change the covariance,
\begin{align*}
    \Cov(\theta_t, \nabla F(\theta_t)) &= \Cov(\theta_t, \nabla^2 F(\theta_\star) (\theta_t - \theta_\star) + \mathsf{Rem}(\theta_t - \theta_\star) ) \\
    &= \Cov(\theta_t, \nabla^2 F(\theta_\star) (\theta_t - \theta_\star)) + \Cov(\theta_t, \mathsf{Rem}(\theta_t - \theta_\star)) \\
    &= \Cov(\theta_t, \nabla^2 F(\theta_\star) \theta_t) + \Cov(\theta_t - \theta_\star, \mathsf{Rem}(\theta_t - \theta_\star)) \\
    &= \Var(\theta_t) \nabla^2 F(\theta_\star) + \Cov(\theta_t - \theta_\star, \mathsf{Rem}(\theta_t - \theta_\star)) \:.
\end{align*}
We now bound $\Cov(\theta_t - \theta_\star, \mathsf{Rem}(\theta_t - \theta_\star))$ as:
\begin{align*}
    \norm{\Cov(\theta_t - \theta_\star, \mathsf{Rem}(\theta_t - \theta_\star))} \leq O( \E[ \norm{\theta_t - \theta_\star}_2^3 ] ) \leq O(\polylog(t)/t^{3/2}) \:.
\end{align*}
Above, the last inequality comes from the high probability bound given in Lemma~\ref{lem:sgd_whp}.
Observing that $\Cov(\theta_t, \nabla F(\theta_t))^\T = \Cov(\nabla F(\theta_t), \theta_t)$,
combining our calculations and continuing from Equation~\eqref{eq:var_recur},
\begin{align*}
    \Var(\theta_{t+1}) &= \Var(\theta_t) + \alpha_t^2 ( \Var(g(\theta_\star;\xi)) + o_t(1) ) - \alpha_t (\Var(\theta_t) \nabla^2 F(\theta_\star)  +\nabla^2 F(\theta_\star) \Var(\theta_t)) \\
    &\qquad + \alpha_t O(\polylog(t)/t^{3/2}) + O(\exp(-t^{\alpha})) \:.
\end{align*}
We now make two observations. Recall that $\alpha_t = 1/(mt)$.
Hence we have $O(\exp(-t^{\alpha})) = \alpha_t^2 O( t^2 \exp(-t^{\alpha}) ) = \alpha_t^2 o_t(1)$.
Similarly, $\alpha_t O(\polylog(t)/t^{3/2}) = \alpha_t^2 O(\polylog(t)/t^{1/2}) = \alpha_t^2 o_t(1)$.
Therefore,
\begin{align*}
    \Var(\theta_{t+1}) &= \Var(\theta_t)- \alpha_t (\Var(\theta_t) \nabla^2 F(\theta_\star)  +\nabla^2 F(\theta_\star) \Var(\theta_t)) + \alpha_t^2 ( \Var(g(\theta_\star;\xi)) + o_t(1) )  \:.
\end{align*}
This matrix recursion can be solved by
Corollary C.1 of~\citet{toulis17},
yielding \eqref{eq:sgd_asymptotic_var}.

To complete the proof, by a Taylor expansion we have:
\begin{align*}
  T \cdot \E[F(\theta_T) - F(\theta_\star)] = \frac{T}{2} \Tr(\nabla^2 F(\theta_\star) \E[(\theta_T-\theta_\star)(\theta_T-\theta_\star)^\T]) + \frac{T}{6} \E[\nabla^3 f(\hat{\theta}) (\theta_T - \theta_\star)^{\otimes 3}] \:.
\end{align*}
As above, we can bound
$\abs{\E[ \nabla^3 f(\hat{\theta}) (\Theta_T - \theta_\star)^{\otimes 3} ]} \leq O( \E[\norm{\theta_T - \theta_\star}_2^3] )\leq O(\polylog(T)/T^{3/2})$,
and hence $T \cdot \abs{\E[ \nabla^3 f(\hat{\theta}) (\Theta_T - \theta_\star)^{\otimes 3} ]} \to 0$.
On the other hand, letting $\mu_T := \E[\theta_T]$, by a bias-variance decomposition,
\begin{align*}
  \E[ (\theta_T - \theta_\star)(\theta_T - \theta_\star)^\T ] &= \E[ (\theta_T - \mu_T)(\theta_T - \mu_T)^\T ] + (\mu_T - \theta_\star)(\mu_T - \theta_\star)^\T \\
  &\succeq \E[ (\theta_T - \mu_T)(\theta_T - \mu_T)^\T ] = \Var(\theta_T) \:.
\end{align*}
Therefore,
\begin{align*}
  T \cdot \E[F(\theta_T) - F(\theta_\star)] \geq \frac{1}{2m} \Tr(\nabla^2 F(\theta_\star) (mT)\Var(\theta_T)) - \frac{T}{6} \abs{\E[\nabla^3 f(\hat{\theta}) (\theta_T - \theta_\star)^{\otimes 3}]} \:.
\end{align*}
Taking limits on both sides yields \eqref{eq:sgd_comparison_bound}.
This concludes the proof of Lemma~\ref{lem:sgd_asymptotics}.

\end{document}